\documentclass{article}

\pdfoutput=1

\usepackage[final,nonatbib]{neurips_2020}

\usepackage{color}
\usepackage[usenames,dvipsnames]{xcolor}
\definecolor{shadecolor}{gray}{0.9}

\newcommand{\blue}[1]{\textcolor{MidnightBlue}{#1}}

\newcommand{\spm}[1]{{\textcolor{gray}{$\ \pm\ $#1}}}

\usepackage{hyperref}
\hypersetup{colorlinks=true}
\hypersetup{linktoc=all}
\hypersetup{citecolor=MidnightBlue}
\hypersetup{linkcolor=MidnightBlue}
\hypersetup{urlcolor=MidnightBlue}
\usepackage[all]{hypcap}

\usepackage{amsmath,amssymb,stmaryrd,amsthm}

\usepackage[nameinlink]{cleveref}
\creflabelformat{equation}{#2\textup{#1}#3}  

\usepackage[utf8]{inputenc} %
\usepackage[T1]{fontenc}    %
\usepackage{url}            %
\usepackage{booktabs}       %
\usepackage{amsfonts}       %
\usepackage{nicefrac}       %
\usepackage{microtype}      %

\usepackage[shortlabels]{enumitem}
\usepackage{multicol, multirow, graphicx}

\usepackage[linesnumbered,ruled,vlined,algo2e]{algorithm2e}
\usepackage{wrapfig,caption,makecell}

\usepackage{ragged2e}

\DeclareTextFontCommand{\emph}{\em}

\newtheorem{theorem}{Theorem}[section]
\newtheorem{lemma}[theorem]{Lemma}
\newtheorem{proposition}[theorem]{Proposition}

\def\randsearch{\texttt{rand\! search}}
\def\bayesopt{\texttt{Bayes\! opt}}
\def\stn{\texttt{STN}}
\def\deepens{\texttt{deep\! ens}}
\def\batchens{\texttt{batch\!~ens}}
\def\hyperens{\texttt{fixed\!~init\!~hyper\!~ens}}
\def\batchhyperens{\blue{\texttt{hyper-batch\!~ens}}}
\def\strathyperens{\blue{\texttt{hyper-deep\!~ens}}}
\def\single{\texttt{single}}

\def\rb{{\mathbf r}}
\def\xb{{\mathbf x}}

\def\bb{{\mathbf b}}

\def\wb{{\mathbf w}}

\def\hb{{\mathbf h}}

\def\sbb{{\mathbf s}}

\def\eb{{\mathbf e}}
\def\ub{{\mathbf u}}
\def\vb{{\mathbf v}}
\def\gb{{\mathbf g}}

\def\xib{{\boldsymbol\xi}}
\def\zetab{{\boldsymbol\zeta}}

\def\deltab{{\boldsymbol\delta}}
\def\Deltab{{\boldsymbol\Delta}}
\def\thetab{{\boldsymbol\theta}}
\def\lambdab{{\boldsymbol\lambda}}

\def\oneb{{\mathbf 1}}
\def\zerob{{\mathbf 0}}

\def\Xb{{\mathbf X}}

\def\Hb{{\mathbf H}}

\def\Wb{{\mathbf W}}

\def\Ib{{\mathbf I}}

\def\Gb{{\mathbf G}}

\def\Cb{{\mathbf C}}

\def\Ub{{\mathbf U}}

\def\Kb{{\mathbf K}}

\def\Sigmab{{\boldsymbol\Sigma}}
\def\Lambdab{{\boldsymbol\Lambda}}
\def\Thetab{{\boldsymbol\Theta}}

\def\Real{{\mathbb{R}}}

\def\Ocal{\mathcal{O}}

\def\Hcal{\mathcal{H}}

\def\Ncal{\mathcal{N}}

\def\Acal{\mathcal{A}}
\def\Dcal{\mathcal{D}}
\def\Ecal{\mathcal{E}}

\def\Pcal{\mathcal{P}}

\def\Mcal{\mathcal{M}}
\def\Scal{\mathcal{S}}

\def\Lcal{\mathcal{L}}

\def\Qcal{\mathcal{Q}}

\def\trace{{\mathrm{Tr}}}

\DeclareMathOperator*{\argmin}{arg\,min}

\def\Exp{{\mathbb{E}}}

\title{Hyperparameter Ensembles for~\\Robustness and Uncertainty Quantification}

\author{%
  Florian Wenzel,\ \ Jasper Snoek,\ \ Dustin Tran,\ \ Rodolphe Jenatton\\
   Google Research \\
   \texttt{$\{$florianwenzel, jsnoek, trandustin, rjenatton$\}$@google.com}\\
}

\begin{document}

\maketitle

\begin{abstract}
Ensembles over neural network weights trained from different random initialization, known as deep ensembles, achieve state-of-the-art accuracy and calibration. The recently introduced batch ensembles provide a drop-in replacement that is more parameter efficient.
In this paper, we design ensembles not only over weights, but over hyperparameters to improve the state of the art in both settings. 
For best performance independent of budget, we propose \textit{hyper-deep ensembles}, a simple procedure that involves a random search over different hyperparameters, themselves stratified across multiple random initializations.
Its strong performance highlights the benefit of combining models with both weight \textit{and} hyperparameter diversity.
We further propose a parameter efficient version, \textit{hyper-batch ensembles}, which builds on the layer structure of batch ensembles and self-tuning networks. The computational and memory costs of our method are notably lower than typical ensembles.
On image classification tasks, with MLP, LeNet, ResNet 20 and Wide ResNet 28-10 architectures, we improve upon both deep and batch 
ensembles.
\end{abstract}

\section{Introduction}
\begin{wrapfigure}{r}{0.43\textwidth}
\vspace{-1cm}
\resizebox{0.42\textwidth}{!}{
\includegraphics{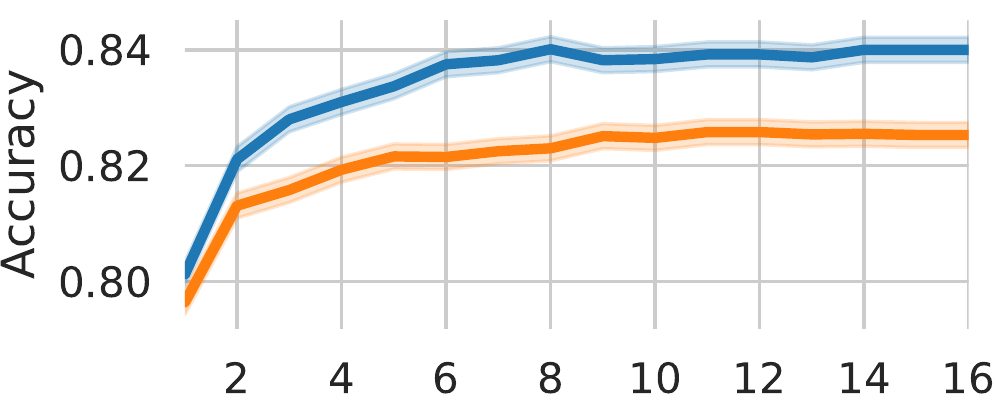}}
\resizebox{0.42\textwidth}{!}{
\includegraphics{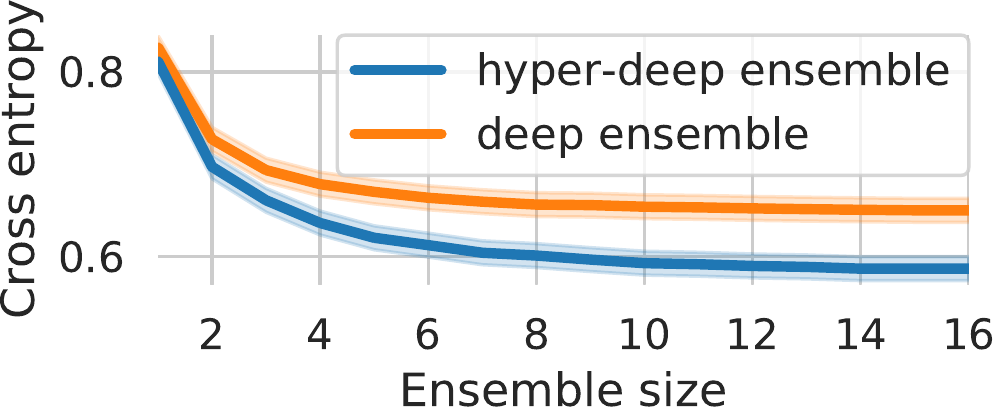}}
\caption{Comparison of our hyper-deep ensemble with deep ensemble for different ensemble sizes
using a Wide ResNet 28-10 over CIFAR-100. Combining models with different hyperparameters is beneficial.}%
\label{fig:str_hyper_ens_cifar100}%
\vspace{-0.0cm}%
\end{wrapfigure}%
Neural networks are well-suited to form \textit{ensembles} of models~\cite{hansen1990neural}.
Indeed, neural networks trained from different random initialization can lead to equally well-performing models that are nonetheless \textit{diverse} in that they make complementary errors on held-out data~\cite{hansen1990neural}.
This property is explained by the multi-modal nature of their loss landscape~\cite{ge2015escaping} and the randomness induced by both their initialization and the stochastic methods commonly used to train them~\cite{Bottou1998, Kingma2014, bottou2018optimization}.

Many mechanisms have been proposed to further foster diversity in ensembles of neural networks, e.g., based on cyclical learning rates~\cite{huang2017snapshot} or Bayesian analysis~\cite{dusenberry2020efficient}. In this paper, we focus on exploiting the diversity induced by combining neural networks defined by different hyperparameters. This concept is already well-established~\cite{Caruana2004} and 
the auto-ML community actively applies it~\cite{Feurer2015a, Snoek2015, Mendoza2016, levesque2016bayesian}.
We build upon this research with the  following two complementary goals. 

First, for performance independent of computational and memory budget, we seek to improve upon \textit{deep ensembles}~\cite{Lakshminarayanan2017}, the current state-of-the-art ensembling method in terms of robustness and uncertainty quantification~\cite{snoek2019can, gustafsson2019evaluating}.
To this end, we develop a simple stratification scheme which combines random search and the greedy selection of hyperparameters from~\cite{Caruana2004} with the benefit of multiple random initializations per hyperparameter like in deep ensembles.
\Cref{fig:str_hyper_ens_cifar100} illustrates our algorithm for a Wide ResNet 28-10 where it leads to substantial improvements, highlighting the benefits of combining different initialization \textit{and} hyperparameters.

Second, we seek to improve upon \textit{batch ensembles}~\cite{wen2020batchensemble}, the current state-of-the-art in efficient ensembles. 
To this end, we propose a 
parameterization combining that of~\cite{wen2020batchensemble} and self-tuning networks~\cite{mackay2019self}, which enables both weight and hyperparameter diversity. Our approach is a drop-in replacement that outperforms batch ensembles and does not need a separate tuning of the hyperparameters.

\subsection{Related work}
    \textbf{Ensembles over neural network weights.} Combining the outputs of several neural networks to improve their single performance has a long history, e.g.,~\cite{levin1990statistical, hansen1990neural, geman1992neural, krogh1995neural, opitz1999popular, dietterich2000ensemble}. Since the quality of an ensemble hinges on the diversity of its members~\cite{hansen1990neural}, many mechanisms were developed to generate diverse ensemble members. For instance, cyclical learning-rate schedules can explore several local minima~\cite{huang2017snapshot, zhang2019cyclical} where ensemble members can be snapshot. Other examples are MC dropout~\cite{gal2016dropout} or the random initialization itself, possibly combined with the bootstrap~\cite{lee2015m, Lakshminarayanan2017}. More generally, Bayesian neural networks can be seen as ensembles with members being weighted by the (approximated) posterior distribution over the parameters~\cite{hinton1993keeping, mackay1995ensemble, neal1995bayesian, blundell2015weight, wenzel2020good, wilson2020bayesian}.

    \textbf{Hyperparameter ensembles.} Hyperparameter-tuning methods~\cite{feurer2019hyperparameter} typically produce a pool of models from which ensembles can be constructed post hoc, e.g.,~\cite{Snoek2015}. This idea has been made systematic as part of \texttt{auto-sklearn}~\cite{Feurer2015a} and successfully exploited in several other contexts, e.g.,~\cite{Feurer2018a} and specifically for neural networks~\cite{Mendoza2016} as well as in computer vision~\cite{saikia2020optimized} and genetics~\cite{Hollerer2020}. In particular, the greedy ensemble construction from~\cite{Caruana2004} (and later variations thereof~\cite{Caruana2006}) was shown to work best among other algorithms, either more expensive or more prone to overfitting.
To the best of our knowledge, such ensembles based on hyperparameters have not been studied in the light of predictive uncertainty. Moreover, we are not aware of existing methods to efficiently build such ensembles, similarly to what batch ensembles do for deep ensembles. Finally, recent research in Bayesian optimization has also focused on directly optimizing the performance of the ensemble while tuning the hyperparameters~\cite{levesque2016bayesian}.

Hyperparameter ensembles also connect closely to probabilistic models over structures. These works often analyze Bayesian nonparametric distributions, such as over depth and width of a neural network, leveraging Markov chain Monte Carlo for inference \cite{kemp2008discovery,adams2010learning,duvenaud2013structure,lake2015human}. In this work, we examine more parametric assumptions, building on the success of variational inference and mixture distributions: for example, the validation step in hyper-batch ensemble can be viewed as a mixture variational posterior and the entropy penalty is the ELBO’s KL divergence toward a uniform prior.

Concurrent to our paper, \cite{zaidi2020neural} construct neural network ensembles within the context of neural architecture search, showing improved robustness for predictions with distributional shift.
One of their methods, NES-RS, has similarities with our hyper-deep ensembles (see \Cref{sec:hyperparameter_ensemble}), also relying on both random search and~\cite{Caruana2004} to form ensembles, but do not stratify over different initializations. We vary the hyperparameters while keeping the architecture fixed while \cite{zaidi2020neural} study the converse.
Furthermore, \cite{zaidi2020neural} do not explore a parameter- and computationally-efficient method (see~\Cref{sec:batch_hyperparameter_ensemble}).

\textbf{Efficient hyperparameter tuning \& best-response function.} Some hyperparameters of a neural network, e.g., its $L_2$ regularization parameter(s), can be optimized by estimating the \textit{best-response} function~\cite{gibbons1992primer}, i.e., the mapping from the hyperparameters to the parameters of the neural networks solving the problem at hand~\cite{brock2017smash}. Learning this mapping is an instance of learning an hypernetwork~\cite{schmidhuber1992learning, schmidhuber1993self, ha2016hypernetworks} and falls within the scope of bilevel optimization problems~\cite{colson2007overview}.
Because of the daunting complexity of this mapping, \cite{lorraine2018stochastic, mackay2019self} proposed scalable local approximations of the best-response function. Similar methodology was also employed for style transfer and image compression~\cite{babaeizadeh2018adjustable, Dosovitskiy2020You}.
The \textit{self-tuning networks} from~\cite{mackay2019self} are an important building block of our approach wherein we extend their setting to the case of an ensemble over different hyperparameters.

\subsection{Contributions}
We examine two regimes to exploit hyperparameter diversity: \textbf{(a)} ensemble performance independent of budget and \textbf{(b)} ensemble performance seeking parameter \textit{efficiency}, where, respectively, deep and batch ensembles \cite{Lakshminarayanan2017, wen2020batchensemble} are state-of-the-art. We propose one ensemble method for each regime:

\textbf{(a) Hyper-deep ensembles.} We define a greedy algorithm to form ensembles of neural networks exploiting two sources of diversity: varied hyperparameters and random  initialization.
By stratifying models with respect to the latter, our algorithm subsumes
deep ensembles that we outperform in our experiments. Our approach is a simple, strong baseline that we hope will be used in future research.

\textbf{(b) Hyper-batch ensembles.} We efficiently construct ensembles of neural networks defined over different hyperparameters. Both the ensemble members \textit{and} their hyperparameters are learned end-to-end in a single training procedure, directly maximizing the ensemble performance. Our approach outperforms batch ensembles and generalizes the layer structure of~\cite{mackay2019self} and~\cite{wen2020batchensemble}, while keeping their original memory compactness and efficient minibatching for parallel training and prediction.

We illustrate the benefits of our two ensemble methods on image classification tasks, with multi-layer perceptron, LeNet, 
ResNet 20 and Wide ResNet 28-10 architectures, in terms of both predictive performance and uncertainty. 
The code for generic hyper-batch ensemble layers can be found in \url{https://github.com/google/edward2} and
the code to reproduce the experiments of \Cref{sec:experiments_resnet} is part of \url{https://github.com/google/uncertainty-baselines}.

\section{Background}

We introduce notation and background required to define our approach.
Consider an i.i.d.~classification setting with data $\Dcal = \{(\xb_n, y_n)\}_{n=1}^N$ where $\xb_n\in \Real^d$ is the feature vector corresponding to the $n$-th example and $y_n$ its class label.
We seek to learn a classifier in the form of a neural network $f_\thetab$ where all its parameters (weights and bias terms) are summarized in $\thetab \in \Real^p$. In addition to its primary parameters $\thetab$, the model $f_\thetab$ will also depend on $m$ hyperparameters that we refer to as $\lambdab \in \Real^m$. For instance, an entry in $\lambdab$ could correspond to the dropout rate of a given layer in $f_\thetab$.

Equipped with some loss function $\ell$, e.g., the cross entropy, and some regularization term $\Omega(\cdot, \lambdab)$, e.g., the squared $L_2$ norm with a strength defined by an entry of $\lambdab$, we are interested in
\begin{equation}\label{eq:map}
    \hat{\thetab}(\lambdab) \in \argmin_{\thetab \in \Real^p}\  
    \Exp_{(\xb, y) \in \Dcal}\big [
    \Lcal(\xb, y, \thetab, \lambdab)
    \big]
    \ \ \ \text{with}\ \ \
    \Lcal(\xb, y, \thetab, \lambdab) = \ell(f_\thetab(\xb, \lambdab), y) + \Omega(\thetab, \lambdab),
\end{equation}
where $\Exp_{(\xb, y) \in \Dcal}[\cdot]$ stands for the expectation with a uniform distribution over $\Dcal$.
As we shall see in \Cref{sec:experiments}, the loss $\ell=\ell_\lambdab$ can also depend on $\lambdab$, for instance to control a label smoothing parameter~\cite{Szegedy2016}. In general, $\lambdab$ is chosen based on some held-out evaluation metric by grid search, random search~\cite{Bergstra2012} or more sophisticated hyperparameter-tuning methods~\cite{feurer2019hyperparameter}.

\subsection{Deep ensembles and batch ensembles}\label{sec:background_deep_ens_and_batch_ens}

Deep ensembles~\cite{Lakshminarayanan2017} are a simple ensembling method where neural networks with different random initialization are combined.
Deep ensembles lead to remarkable predictive performance and robust uncertainty estimates~\cite{snoek2019can, gustafsson2019evaluating}.
Given some hyperparameters $\lambdab_0$, a deep ensemble of size $K$ amounts to solving $K$ times (\ref{eq:map}) with random initialization and aggregating the outputs of $\{f_{\hat{\thetab}_k(\lambdab_0)}(\cdot, \lambdab_0)\}_{k=1}^K$.

Batch ensembles~\cite{wen2020batchensemble} are a state-of-the-art \textit{efficient} alternative to deep ensembles, preserving their performance while reducing their computational and memory burden.
To simplify the presentation, we focus on the example of a dense layer in $f_\thetab$, with weight matrix $\Wb \in \Real^{r \times s}$ where $r$ and $s$ denote the input and output dimensions of the layer respectively.

A deep ensemble of size $K$ needs to train, predict with, and store $K$ weight matrices $\{\Wb_k\}_{k=1}^K$. Instead, batch ensembles consider a \textit{single} matrix $\Wb \in \Real^{r \times s}$ together with two sets of auxiliary vectors $[\rb_1,\dots, \rb_K] \in \Real^{r \times K}$ and $[\sbb_1,\dots, \sbb_K] \in \Real^{s \times K}$ such that the role of $\Wb_k$ is played by
\begin{equation}\label{eq:batch_ensemble}
    \Wb \circ (\rb_k \sbb_k^\top)\ \ \text{for each}\ \ k \in \{1,\dots, K\},
\end{equation}
where we denote by $\circ$ the element-wise product (which we will broadcast row-wise or column-wise depending on the shapes at play). Not only does (\ref{eq:batch_ensemble}) lead to a memory saving, but it also allows for efficient minibatching, where each datapoint may use a different ensemble member. Given a batch of inputs $\Xb \in \Real^{b \times r}$,
the predictions for the $k$-th member equal $\Xb [\Wb \circ (\rb_k \sbb_k^\top)] = [(\Xb \circ \rb_k^\top) \Wb] \circ \sbb_k^\top$. By properly tiling the batch $\Xb$, the $K$ members 
can thus predict in parallel in one forward pass~\cite{wen2020batchensemble}.

\subsection{Self-tuning networks}\label{sec:background_stn}
Hyperparameter tuning typically involves multiple runs of the training procedure.
One efficient alternative \cite{lorraine2018stochastic,mackay2019self} is to approximate the best-response function, i.e., the mapping from $\lambdab$ to optimal parameters $\hat{\thetab}(\lambdab)$.
The local approximation of \cite{mackay2019self} captures the changes of $\lambdab$ by scaling and shifting the hidden units of $f_\thetab$, which requires in turn extra parameters  $\thetab' \in \Real^{p'}$, summarized in $\Thetab = \{\thetab, \thetab'\}$. \cite{mackay2019self} call the resulting approach \textit{self-tuning network} since $f_\Thetab$ tunes online its own hyperparameters $\lambdab$.
In the sequel, $\lambdab$ will be continuous such as dropout rates, $L_2$ penalties and label smoothing.

\paragraph{Example of the dense layer.} We illustrate the choice and role of $\thetab'$ in the example of a dense layer (the convolutional layer is similar to~\cite{perez2018film}; see details in~\cite{mackay2019self}). The weight matrix $\Wb \in \Real^{r \times s}$ and bias $\bb \in \Real^s$ of a dense layer are defined as (with $\Deltab$ and $\deltab$ of the same shapes as $\Wb$ and $\bb$ respectively),
\begin{equation}\label{eq:dense_layer_stn}
\Wb(\lambdab) = \Wb + \Delta \circ \eb(\lambdab)^\top 
\ \ \text{and}\ \
\bb(\lambdab) = \bb + \deltab \circ \eb'(\lambdab), 
\end{equation}
where $\eb(\lambdab) \in \Real^s$ and $\eb'(\lambdab) \in \Real^s$ are real-valued embeddings of $\lambdab$.
In~\cite{mackay2019self}, the embedding is linear, i.e., $\eb(\lambdab)=\Cb \lambdab$ and $\eb'(\lambdab)=\Cb' \lambdab$. In this example, we have original parameters $\thetab =\{\Wb, \bb\}$ as well as the additional parameters $\thetab' = \{\Deltab, \deltab, \Cb, \Cb'\}$.
\begin{figure}[t]
\vspace{-0.0cm}
\resizebox{\textwidth}{!}{
\includegraphics[scale=0.38]{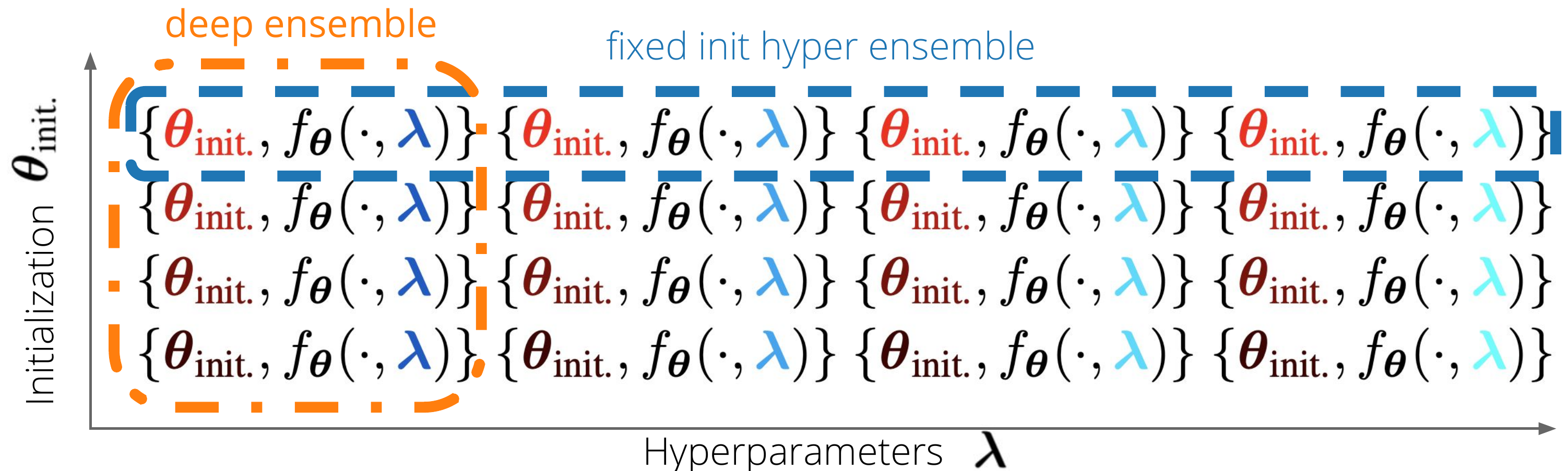}
\includegraphics[scale=0.46]{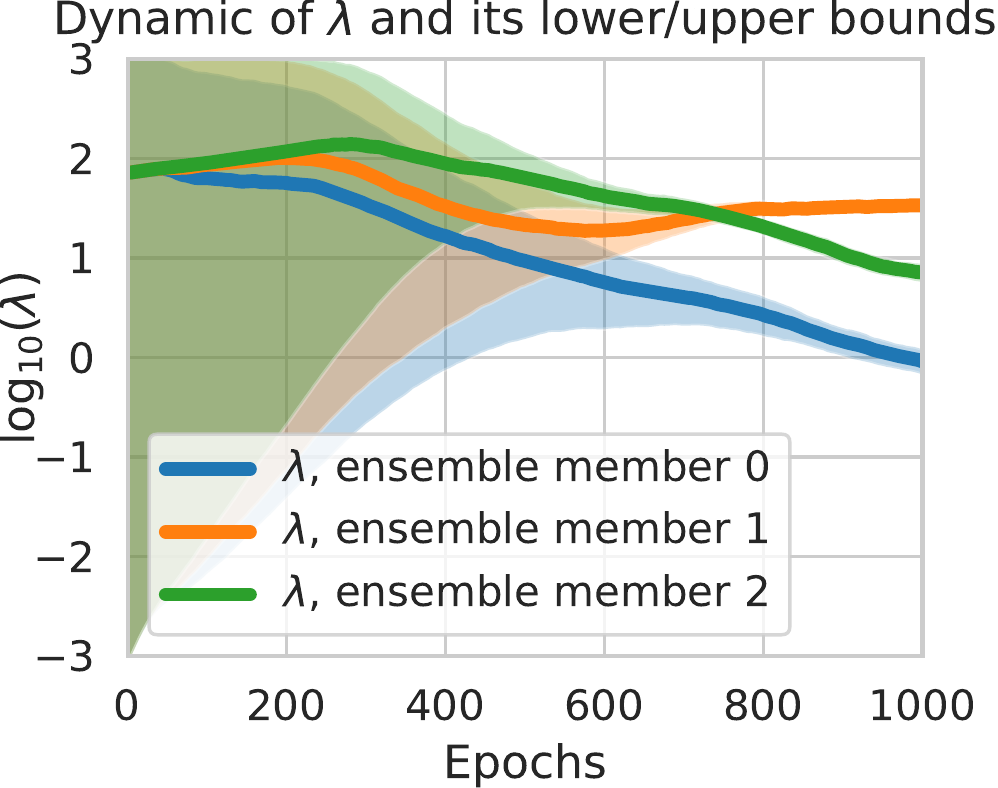}
}
\vspace{-0.5cm}%
\caption{\textsc{left:} Pictorial view of deep ensemble (``column'') and fixed init hyper ensemble (``row'') for models $f_\thetab(\cdot,\lambdab)$ with parameters $\thetab$ and hyperparameters $\lambdab$. Our new method \textit{hyper-deep ensemble} can search in the whole ``block'', exploiting both initialization and hyperparameter diversity. \textsc{right:} Example of the optimization path of \textit{hyper-batch ensemble} for an entry of the hyperparameters $\lambdab$ (the $L_2$ parameter of an MLP over CIFAR-100) with its upper/lower bounds (shaded regions).
The lower/upper bounds of the three members converge to a diverse set of hyperparameters.
}%
\label{fig:initialization_vs_lambdas_and_example_trajectory}%
\vspace{-0.0cm}%
\end{figure}%
\paragraph{Training objective.} Since $\thetab'$ captures changes in $\thetab$ induced by changes in $\lambdab$, \cite{lorraine2018stochastic, mackay2019self} replace the typical objective (\ref{eq:map}), defined for a \textit{single} value of $\lambdab$, with an \textit{expected} objective \cite{lorraine2018stochastic, mackay2019self,Dosovitskiy2020You},
\begin{equation}\label{eq:global_stn}
    \min_{\Thetab \in \Real^{p+p'}}
\Exp_{\substack{\lambdab \sim p(\lambdab), (\xb, y) \in \Dcal}} \big[
\Lcal(\xb, y, \Thetab, \lambdab)
\big],
\end{equation}
where $p(\lambdab)$ denotes some distribution over the hyperparameters $\lambdab$. When $p$ is kept fixed during the optimization of~(\ref{eq:global_stn}), the authors of~\cite{lorraine2018stochastic} observed that $\hat{\thetab}(\lambdab)$ is not well approximated and proposed instead to use a distribution $p_t(\lambdab)=p(\lambdab|\xib_t)$ varying with the iteration $t$.
In our work we choose $p(\cdot|\xib_t)$ to be a log-uniform distribution with $\xib_t$ containing the bounds of the ranges of $\lambdab$ (see \Cref{sec:batch_hyperparameter_ensemble}). 
The key benefit from~(\ref{eq:global_stn}) is that a \textit{single} (though, more costly) training gives access to a mapping $\lambdab \mapsto f_{\hat{\Thetab}}(\cdot, \lambdab)$ which approximates the behavior of $f_{\hat{\Thetab}}$ for hyperparameters in the support of $p(\lambdab)$.

\paragraph{Alternating optimization.} The procedure followed by~\cite{mackay2019self} consists in alternating between training and tuning steps. First, the training step performs a stochastic gradient update of $\Thetab$ in~(\ref{eq:global_stn}), jointly sampling $\lambdab\sim p(\lambdab|\xib_t)$ and $(\xb, y) \in \Dcal$. Second, the tuning step makes a stochastic gradient update of $\xib_t$ by minimizing some \textit{validation} objective (e.g., the cross entropy):
\begin{equation}\label{eq:validation_loss}
\min_{\xib_t}
\Exp_{\substack{\lambdab \sim p(\lambdab|\xib_t), (\xb, y) \in \Dcal_\text{val}}} \big[ \ell_\text{val}(f_\Thetab(\xb, \lambdab), y)\big].
\end{equation}
In~(\ref{eq:validation_loss}), derivatives are taken through samples $\lambdab\sim p(\lambdab|\xib_t)$ by applying the reparametrization trick~\cite{kingma2013auto}.
To prevent $p(\lambdab|\xib_t)$ from collapsing to a degenerate distribution, and inspired by variational inference, the authors of~\cite{mackay2019self} add an entropy regularization term $\Hcal[\cdot]$ controlled by $\tau \geq 0$ so that~(\ref{eq:validation_loss}) becomes
\begin{equation}\label{eq:validation_loss_and_entropy}
\min_{\xib_t}
\Exp_{\substack{\lambdab \sim p(\lambdab|\xib_t), (\xb, y) \in \Dcal_\text{val}}} \big[  \ell_\text{val}(f_\Thetab(\xb, \lambdab), y) - \tau \Hcal[p(\lambdab|\xib_t)]
\big].
\end{equation}
\section{Hyper-deep ensembles}\label{sec:hyperparameter_ensemble}

\Cref{fig:initialization_vs_lambdas_and_example_trajectory}-(left) visualizes different models $f_\thetab(\cdot, \lambdab)$ according to their hyperparameters $\lambdab$ along the $x$-axis and their initialization $\thetab_\text{init.}$ on the $y$-axis.
In this view, a deep ensemble corresponds to a ``column'' where models with different random initialization are combined together, for a fixed $\lambdab$. On the other hand, a ``row'' corresponds to the combination of models with different hyperparameters. Such a ``row'' typically stems from the application of some hyperparameter-tuning techniques~\cite{feurer2019hyperparameter}.

\paragraph{Fixed initialization hyper ensembles.} Given the simplicity, broad applicability, and performance of the greedy algorithm from~\cite{Caruana2004}---e.g., in auto-ML settings~\cite{Feurer2015a}, we use it as our canonical procedure to generate a ``row'', i.e., an ensemble of neural networks with fixed parameter initialization and various hyperparameters. We refer to it as \textit{fixed init hyper ensemble}. For completeness, we recall the procedure from~\cite{Caruana2004} in~\Cref{sec:supp_mat_details_about_stratified_hyper_ens} (Algorithm~\ref{alg:hyper_ens}, named \texttt{hyper\_ens}). Given an input set of models (e.g., from random search), \texttt{hyper\_ens} greedily grows an ensemble until some target size $K$ is met by selecting the model with the best improvement of some score, e.g., the validation log-likelihood. We select the models \textit{with replacement} to be able to learn weighted combinations thereof (see Section~2.1 in~\cite{Caruana2004}).
Note that the procedure from~\cite{Caruana2004} does not require the models to have a fixed initialization: we consider here a fixed initialization to isolate the effect of just varying the hyperparameters (while deep ensembles vary only the initialization, with fixed hyperparameters).

Our goal is two-fold: (a) we want to demonstrate the complementarity of random initialization and hyperparameters  as sources of diversity in the ensemble, and (b) 
design a simple algorithmic scheme that exploits both sources of diversity while encompassing the construction of deep ensembles as a subcase.
We defer to~\Cref{sec:experiments} the study of (a) and next focus on (b).
\paragraph{Hyper-deep ensembles.} 
We proceed in three main steps, as summarized in Algorithm~\ref{alg:stratified_hyper_ens}. In lines 1-2, we first generate one ``row'' according to \texttt{hyper\_ens} based on the results of random search~\cite{Bergstra2012} as input. We then tile and stratify that ``row'' by training the models for different random initialization (see lines 4-7).
The resulting set of models is illustrated in \Cref{fig:initialization_vs_lambdas_and_example_trajectory}-(left).
In line 10, we finally re-apply \texttt{hyper\_ens} on that stratified set of models to extract an ensemble that can exploit the two sources of diversity. By design, a deep ensemble is one possible outcome of this procedure---one ``column''---and so is \textit{fixed init hyper ensemble} described in the previous paragraph---one ``row''. 

\begin{wrapfigure}{r}{0.43\textwidth}
\hspace*{0.01cm}
\resizebox{0.42\textwidth}{!}{
\begin{minipage}{0.5\textwidth}
\vspace*{-0.8cm}
\begin{algorithm2e}[H]
\SetAlgoLined
 $\Mcal_0=\{f_{\thetab_j}(\cdot, \lambdab_j)\}_{j=1}^\kappa \!\xleftarrow{}$ \texttt{rand\_search}($\kappa$)\;
 $\Ecal_0 \xleftarrow{}$ \texttt{hyper\_ens}($\Mcal_0,\ K$)\ 
 and\ $\Ecal_\text{strat.} = \{\ \}$\;
 \ForEach{$f_{\thetab}(\cdot, \lambdab) \in \Ecal_0.\mathrm{{unique}()}$}{
    \ForEach{$k \in \{1,\dots, K\}$}{
        $\thetab' \xleftarrow{}$ random initialization\;
        $f_{\thetab_k}(\cdot, \lambdab) \xleftarrow{}$ train
        $f_{\thetab'}(\cdot, \lambdab)$\;
        $\Ecal_\text{strat.} = \Ecal_\text{strat.} \cup \{\ f_{\thetab_k}(\cdot, \lambdab) \}$\;
    }
 }
 \textbf{return} \texttt{hyper\_ens}($\Ecal_\text{strat.},\ K$)\;
\caption{\texttt{hyper\_deep\_ens}($K, \kappa$)}
\label{alg:stratified_hyper_ens}
\end{algorithm2e}
\end{minipage}
}
\end{wrapfigure}
In lines 1-2, running random search leads to a set of $\kappa$ models (i.e., $\Mcal_0$). If we were to stratify all of them, we would need $K$ seeds for each of those $\kappa$ models, hence a total of $\Ocal(\kappa K)$ models to train. However, we first apply \texttt{hyper\_ens} to extract $K$ models out of the $\kappa$ available ones, with $K \ll \kappa$. The stratification then needs $K$ seeds for each of those $K$ models (lines 4-7), thus $\Ocal(K^2)$ models to train.
We will see in~\Cref{sec:experiments} that even with standard hyperparameters, e.g., dropout or $L_2$ parameters, Algorithm~\ref{alg:stratified_hyper_ens} can lead to substantial improvements over deep ensembles.
In~\Cref{sec:supp_mat_stratified_hyper_ens_ablation_top_k},
we conduct ablation studies to relate to the top-$K$ strategy used in~\cite{saikia2020optimized} and NES-RS from~\cite{zaidi2020neural}.
\section{Hyper-batch ensembles}\label{sec:batch_hyperparameter_ensemble}
This section presents our efficient approach to construct ensembles over different hyperparameters.

\subsection{Composing the layer structures of batch ensembles and self-tuning networks}\label{sec:batch_hyperparameter_ensemble_layer}

The core idea lies in the \textit{composition} of the layers used by batch ensembles~\cite{wen2020batchensemble} for ensembling parameters and self-tuning networks~\cite{mackay2019self} for parameterizing the layer as an explicit function of hyperparameters. The composition preserves complementary features from both approaches.

We continue the example of the dense layer from 
\Cref{sec:background_deep_ens_and_batch_ens}-\Cref{sec:background_stn}.
The convolutional layer is described in~\Cref{sec:supp_mat_conv2d_batch_hyperparameter_ensemble}. Assuming an ensemble of size $K$, we have for $k \in \{1,\dots, K\}$
\begin{equation}\label{eq:dense_layer_batch_stn}
\Wb_k(\lambdab_k) =
\Wb \circ (\rb_k \sbb_k^\top)
+
[\Delta \circ (\ub_k \vb_k^\top) ] \circ \eb(\lambdab_k)^\top 
\ \ \text{and}\ \
\bb_k(\lambdab_k) = 
\bb_k + \deltab_k \circ \eb'(\lambdab_k),
\end{equation}
where the
$\rb_k$'s (respectively, $\ub_k$'s) in $\Real^{r}$ and $\sbb_k$'s (respectively, $\vb_k$'s) in $\Real^{s}$ are vectors which diversify the shared matrix $\Wb$ (respectively, $\Deltab$) in $\Real^{r \times s}$;
and the $\bb_k$'s in $\Real^{s}$ and $\deltab_k$'s in $\Real^{s}$ are the bias terms for each of the $K$ ensemble members.
We comment on some important properties of~(\ref{eq:dense_layer_batch_stn}):
\begin{itemize}%
    \item As noted by \cite{wen2020batchensemble}, formulation~(\ref{eq:batch_ensemble}) includes a set of rank-1 factors which diversify individual ensemble member weights. In (\ref{eq:dense_layer_batch_stn}), the rank-1 factors $\rb_k \sbb_k^\top$ and $\ub_k \vb_k^\top$ capture this weight diversity for each respective term.
    \item
    As noted by~\cite{mackay2019self}, formulation~(\ref{eq:dense_layer_stn}) captures local hyperparameter variations in the vicinity of some $\lambdab$. The term $[\Delta \circ (\ub_k \vb_k^\top) ] \circ \eb(\lambdab_k)^\top$ in (\ref{eq:dense_layer_batch_stn}) extends this behavior to the vicinity of the $K$ hyperparameters $\{\lambdab_1,\dots, \lambdab_K\}$ indexing the $K$ ensemble members.
    \item Equation~(\ref{eq:dense_layer_batch_stn}) maintains the compactness of the original layers of~\cite{mackay2019self, wen2020batchensemble}
    with a resulting memory footprint about twice as large as~\cite{wen2020batchensemble} and equivalent to~\cite{mackay2019self} up to the rank-1 factors.
    \item Given $K$ hyperparameters $\{\lambdab_1,\dots, \lambdab_K\}$ and a batch of inputs $\Xb \in \Real^{b \times r}$, the structure of~(\ref{eq:dense_layer_batch_stn}) preserves the efficient minibatching of~\cite{wen2020batchensemble}. If $\oneb_b$ is the vector of ones in $\Real^b$, we can tile $\Xb$, $\oneb_b \lambdab_k^\top$ and $\oneb_b \eb(\lambdab_k)^\top\!$, enabling all $K$ members to predict in a \textit{single} forward pass.
    \item
    From an implementation perspective, (\ref{eq:dense_layer_batch_stn}) enables direct reuse of existing code, e.g., \texttt{DenseBatchEnsemble} and \texttt{Conv2DBatchEnsemble} from~\cite{tran2019bayesian}. The implementation of our layers can be found in \url{https://github.com/google/edward2}.
\end{itemize}

\subsection{Objective function: from single model to  ensemble}\label{sec:batch_hyperparameter_ensemble_objective}

We first need to slightly overload the notation from \Cref{sec:background_stn} and we write $f_\Thetab(\xb, \lambdab_k)$ to denote the prediction for the input $\xb$ of the $k$-th ensemble member indexed by $\lambdab_k$.
In $\Thetab$, we pack all the parameters of $f$, as those described in the example of the dense layer in \Cref{sec:batch_hyperparameter_ensemble_layer}. In particular, predicting with $\lambdab_k$ is understood as using the corresponding parameters $\{\Wb_k(\lambdab_k), \bb_k(\lambdab_k)\}$ in~(\ref{eq:dense_layer_batch_stn}). 

\paragraph{Training and validation objectives.}
We want the ensemble members to account for a diverse combination of hyperparameters. As a result, each ensemble member is assigned its \textit{own} distribution of hyperparameters, which we write $p_t(\lambdab_k) = p(\lambdab_k|\xib_{k,t})$ for $k \in \{1,\dots, K\}$. 
Along the line of~(\ref{eq:global_stn}), we consider an expected training objective which now simultaneously operates over $\Lambdab_K = \{\lambdab_k\}_{k=1}^K$
\begin{equation}\label{eq:batch_hyperparameter_ensemble_training}
    \min_{\Thetab}
\Exp_{\substack{\Lambdab_K \sim q_t,  (\xb, y) \in \Dcal}} \Big[
\Lcal(\xb, y, \Thetab, \Lambdab_K)
\Big]
\ \ \text{with}\ \
q_t\big(\Lambdab_K\big) = q(\Lambdab_K| \{\xib_{k,t}\}_{k=1}^K) =
\prod_{k=1}^K p_t(\lambdab_k)
\end{equation}
and where $\Lcal$, compared with~(\ref{eq:map}), is extended to handle the ensemble predictions
\begin{equation*}
\Lcal(\xb, y, \Thetab, \Lambdab_K)
=
\ell\big(\{f_\Thetab(\xb, \lambdab_k)\}_{k=1}^K, y\big)
+
\Omega\big(\Thetab, \{\lambdab_k\}_{k=1}^K\big).
\end{equation*}
For example, the loss $\ell$ can be the ensemble cross entropy or the average
ensemble-member cross entropy (in our experiments, we will use the latter as recent results suggests it often generalizes better \cite{dusenberry2020efficient}).
The introduction of one distribution $p_t$ per ensemble member also affects the validation step of the alternating optimization, in particular we adapt~(\ref{eq:validation_loss_and_entropy}) to become
\begin{equation}\label{eq:batch_hyperparameter_ensemble_validation}
\min_{\{\xib_{k,t}\}_{k=1}^K}
\Exp_{\substack{\Lambdab_K \sim q_t,  (\xb, y) \in \Dcal_\text{val}}} 
\Big[ 
\ell_\text{val}(\{f_\Thetab(\xb, \lambdab_k)\}_{k=1}^K, y) - \tau \Hcal\big[q_t\big(\Lambdab_K\big)\big]
\Big].
\end{equation}
Note that the extensions~(\ref{eq:batch_hyperparameter_ensemble_training})-(\ref{eq:batch_hyperparameter_ensemble_validation}) with $K=1$ fall back to the standard formulation of~\cite{mackay2019self}. 
In our experiments, we take $\Omega$ to be $L_2$ regularizers applied to the parameters $\Wb_k(\lambdab_k)$ and $\bb_k(\lambdab_k)$ of each ensemble member.
In~\Cref{sec:supp_mat_l2_reg_batch_hyperparameter_ensemble}, we show how to efficiently vectorize the computation of $\Omega$ across the ensemble members and mini-batches of $\{\lambdab_k\}_{k=1}^K$ sampled from $q_t$, as required by~(\ref{eq:batch_hyperparameter_ensemble_training}).
In practice, we use one sample of $\Lambdab_K$ for each data point in the batch: for MLP/LeNet (\Cref{sec:experiments_mlp_lenet}), we use 256, while for ResNet-20/W.~ResNet-28-10 (\Cref{sec:experiments_resnet}), we use 512 (64 for each of 8 workers).

\paragraph{Definition of $p_t$.}
In the experiments of \Cref{sec:experiments}, we will manipulate hyperparameters $\lambdab$ that are positive and bounded (e.g., a dropout rate).
For each ensemble member with hyperparameters $\lambdab_k \in \Real^m$, we thus define its distribution $p_t(\lambdab_k)=p(\lambdab_k|\xib_{k,t})$ to be $m$ independent \textit{log-uniform distributions} (one per dimension in $\lambdab_k$), which is a standard choice for hyperparameter tuning, e.g.,~\cite{Bergstra2011,Bergstra2012,Mendoza2016}. With this choice, $\xib_{k,t}$ contains $2m$ parameters, namely the bounds of the ranges of the $m$ distributions. Similar to~\cite{mackay2019self}, at prediction time, we  take $\lambdab_k$ to be equal to the means $\lambdab^{\text{mean}}_k$ of the distributions $p_t(\lambdab_k)$.
In~\Cref{sec:supp_mat_choice_p_t_batch_hyperparameter_ensemble}, we provide additional details about $p_t$.

The validation steps~(\ref{eq:validation_loss_and_entropy}) and~(\ref{eq:batch_hyperparameter_ensemble_validation}) seek to optimize the bounds of the ranges. More specifically, the loss $\ell_\text{val}$ favors compact ranges around a good hyperparameter value whereas the entropy term encourages wide ranges, as traded off by $\tau$. We provide an example of the optimization trajectory of $\lambda$ and its range in \Cref{fig:initialization_vs_lambdas_and_example_trajectory}-(right), where $\lambda$ corresponds to the mean of the log-uniform distribution.
\section{Experiments}\label{sec:experiments}
Throughout the experiments, we use both metrics that depend on the predictive uncertainty---negative log-likelihood (NLL)
and expected calibration error (ECE)~\cite{naeini2015obtaining}---and metrics that do not, e.g., the classification accuracy. 
The supplementary material also reports Brier score~\cite{brier1950verification} (for which we typically observed a strong correlation with NLL).
Moreover, as diversity metric, we take the predictive disagreement of the ensemble members normalized by (1-accuracy), as used in~\cite{fort2019deep}.
In the tables, we write the number of ensemble members in brackets ``($\cdot$)'' next to the name of the methods.

\subsection{Multi-layer perceptron and LeNet on Fashion MNIST \& CIFAR-100}\label{sec:experiments_mlp_lenet}
To validate our approaches and run numerous ablation studies, we first focus on small-scale models, namely MLP
and LeNet~\cite{lecun1990handwritten}, over CIFAR-100~\cite{krizhevsky2009learning} and Fashion MNIST~\cite{xiao2017fashion}.
For both models, we add a dropout layer~\cite{srivastava2014dropout} before their last layer. For each pair of dataset/model type, we consider two tuning settings involving the dropout rate and different $L_2$ regularizers defined with varied granularity, e.g., layerwise.
\Cref{sec:supp_mat_arch_mlp_lenet} gives all the details about the training, tuning and dataset definitions.
\setlength{\tabcolsep}{4pt}
\begin{table}[t]
\caption{Comparison over CIFAR-100 and Fashion MNIST with MLP and LeNet models. We report means $\pm$ standard errors (over the 3 random seeds and pooled over the 2 tuning settings). ``single'' stands for the best between \randsearch{} and \bayesopt{}. ``fixed init ens'' is a shorthand for \hyperens{}, i.e., a ``row'' in~\Cref{fig:initialization_vs_lambdas_and_example_trajectory}-(left). 
We separately compare the \textit{efficient} methods (3 rightmost columns) and we mark in bold the best results (within one standard error).  
Our two methods {\color{MidnightBlue}hyper-deep/hyper-batch ensembles} improve upon deep/batch ensembles respectively (in~\Cref{sec:supp_mat_statistical_significance_mlp_lenet}, we assess the statistical significance of those improvements with a Wilcoxon signed-rank test, paired along settings, datasets and model types).}
\label{tab:mlp_lenet_results_ens_size_3_without_brier}
\resizebox{\textwidth}{!}{%
\begin{tabular}{@{}clcccc||ccc@{}}
\toprule
 &  & single~(1) & fixed init ens~(3) & {\color{MidnightBlue}hyper-deep ens~(3)} & deep ens~(3) & batch ens~(3) & STN (1) & {\color{MidnightBlue}hyper-batch ens~(3)} \\ \midrule \midrule
\multirow{3}{*}{\shortstack{cifar100\\(mlp)}} 
 & nll $\ \,\downarrow$ & {2.977}\spm{0.010} & \textbf{2.943}\spm{0.010} & \textbf{2.953}\spm{0.058} & \textbf{2.969}\spm{0.057} & 3.015\spm{0.003} & 3.029\spm{0.006} & \textbf{2.979}\spm{0.004}  \\
 & acc $\,\uparrow$ & {0.277}\spm{0.002} & \textbf{0.287}\spm{0.003} & \textbf{0.291}\spm{0.004} & \textbf{0.289}\spm{0.003} & 0.275\spm{0.001} & 0.268\spm{0.002} & \textbf{0.281}\spm{0.002}  \\
 & ece $\,\downarrow$ &  \textbf{0.034}\spm{0.008} & \textbf{0.029}\spm{0.007} & \textbf{0.022}\spm{0.007} & \textbf{0.038}\spm{0.014} & \textbf{0.022}\spm{0.002} & 0.033\spm{0.004} & 0.030\spm{0.002}   \\ \cmidrule(l){1-9}
\multirow{3}{*}{\shortstack{cifar100\\(lenet)}} 
 & nll $\ \,\downarrow$ & \textbf{2.399}\spm{0.204} & \textbf{2.259}\spm{0.067} & \textbf{2.211}\spm{0.066} & \textbf{2.334}\spm{0.141} & 2.350\spm{0.024} & 2.329\spm{0.017} & \textbf{2.283}\spm{0.016}\\
 & acc $\,\uparrow$ &  0.420\spm{0.011} & \textbf{0.439}\spm{0.008} & \textbf{0.452}\spm{0.007} & \textbf{0.421}\spm{0.026} & \textbf{0.438}\spm{0.003} & 0.415\spm{0.003} & 0.428\spm{0.003} \\
 & ece $\,\downarrow$ & \textbf{0.064}\spm{0.036} & \textbf{0.049}\spm{0.023} & \textbf{0.039}\spm{0.013} & \textbf{0.050}\spm{0.015} & 0.058\spm{0.015} & \textbf{0.024}\spm{0.007} & 0.058\spm{0.004} \\
 \cmidrule(l){1-9}
\multirow{3}{*}{\shortstack{fmnist\\(mlp)}}
 & nll $\ \,\downarrow$  &  0.323\spm{0.003} & \textbf{0.312}\spm{0.003} & \textbf{0.310}\spm{0.001} & 0.319\spm{0.005} & 0.351\spm{0.004} & 0.316\spm{0.003} & \textbf{0.308}\spm{0.002} \\ 
 & acc $\,\uparrow$ & 0.889\spm{0.002} & \textbf{0.893}\spm{0.001} & \textbf{0.895}\spm{0.001} & 0.889\spm{0.003} & 0.884\spm{0.001} & \textbf{0.890}\spm{0.001} & \textbf{0.892}\spm{0.001} \\
 & ece $\,\downarrow$ &  \textbf{0.013}\spm{0.003} & \textbf{0.012}\spm{0.005} & \textbf{0.014}\spm{0.003} & \textbf{0.010}\spm{0.003} & 0.020\spm{0.001} & \textbf{0.016}\spm{0.001} & \textbf{0.016}\spm{0.001} \\ \cmidrule(l){1-9}
\multirow{3}{*}{\shortstack{fmnist\\(lenet)}}
 & nll $\ \,\downarrow$ & 0.232\spm{0.002} & \textbf{0.219}\spm{0.002} & \textbf{0.216}\spm{0.002} & 0.226\spm{0.004} & 0.230\spm{0.005} & 0.224\spm{0.003} & \textbf{0.212}\spm{0.001} \\
 & acc $\,\uparrow$ &  0.919\spm{0.001} & \textbf{0.924}\spm{0.001} & \textbf{0.926}\spm{0.002} & 0.920\spm{0.002} & 0.920\spm{0.001} & 0.920\spm{0.001} & \textbf{0.924}\spm{0.001} \\
 & ece $\,\downarrow$ &  \textbf{0.017}\spm{0.005} & \textbf{0.014}\spm{0.004} & \textbf{0.018}\spm{0.002} & \textbf{0.013}\spm{0.004} & 0.017\spm{0.002} & 0.015\spm{0.001} & \textbf{0.009}\spm{0.001} \\  \bottomrule
\end{tabular}%
}
\end{table}

\paragraph{Baselines.} We compare our methods (i)
\strathyperens{}: hyper-deep ensemble of \Cref{sec:hyperparameter_ensemble} and
(ii) \batchhyperens{}: hyper-batch ensemble of \Cref{sec:batch_hyperparameter_ensemble}, to
(a)  \randsearch{}: the best single model after 50 trials of random search~\cite{Bergstra2012},
(b) \bayesopt{}: the best single model after 50 trials of Bayesian optimization~\cite{Snoek2012, Golovin2017},
(c) \deepens{}: deep ensemble~\cite{Lakshminarayanan2017} using the best hyperparameters found by random search,
(d) \batchens: batch ensemble~\cite{wen2020batchensemble},
(e) \stn: self-tuning networks~\cite{mackay2019self},
and
(f) \hyperens{}: %
defined in \Cref{sec:hyperparameter_ensemble}.
 The supplementary material details how we tune the hyperparameters specific to \batchens{}, \stn{} and \batchhyperens{} (see \Cref{sec:supp_mat_hyperparameters_batch_ens}, \Cref{sec:supp_mat_hyperparameters_stn} and~\Cref{sec:supp_mat_hyperparameters_batch_stn_ens} and further ablations about $\eb$ in~\Cref{sec:supp_mat_ablation_e} and $\tau$ in~\Cref{sec:supp_mat_ablation_tau}).
Note that \batchens{} needs the tuning of its own hyperparameters and those of the MLP/LeNet models, while \stn{} and \batchhyperens{} automatically tune the latter.

We highlight below the key conclusions from~\Cref{tab:mlp_lenet_results_ens_size_3_without_brier} with single models and ensemble of sizes 3. The same conclusions can also be drawn for the ensemble of size 5 (see~\Cref{sec:supp_mat_ens_size_3_and_5}).

\paragraph{Ensembles benefit from both weight and hyperparameter diversity.}
With the pictorial view of~\Cref{fig:initialization_vs_lambdas_and_example_trajectory} in mind, \hyperens{}, i.e., a ``row'', tends to outperform \deepens, i.e., a ``column''. Moreover, those two approaches (as well as the other methods of the benchmark) are outperformed by our stratified procedure \strathyperens{}, demonstrating the benefit of combining hyperparameter and initialization diversity (see~\Cref{sec:supp_mat_statistical_significance_mlp_lenet} for the detailed assessment of the statistical significance). In~\Cref{sec:supp_mat_mlp_lenet_diversity_analysis}, we study more specifically the diversity and we show that \strathyperens{} has indeed more diverse predictions than~\deepens{}.

\paragraph{Efficient ensembles benefit from both weight and hyperparameter diversity.}  Among the efficient approaches (the three rightmost columns of~\Cref{tab:mlp_lenet_results_ens_size_3_without_brier}), \batchhyperens{} performs best. It improves upon both \stn{} and \batchens{}, the two methods it builds upon.
In line with~\cite{mackay2019self}, \stn{} typically matches or improves upon \randsearch{} and \bayesopt{}.
As explained in \Cref{sec:batch_hyperparameter_ensemble_layer}, \batchhyperens{} has however twice the number of parameters of \batchens{}. In \Cref{sec:supp_mat_deep_batch_ens}, we thus compare with a ``deep ensemble of two batch ensembles'' (i.e., resulting in the same number of parameters but twice as many members as for \batchhyperens{}).
In that case, \batchhyperens{} also either improves upon or matches the performance of the combination of two \batchens{}.

\subsection{ResNet-20 and Wide ResNet-28-10 on CIFAR-10 \& CIFAR-100}\label{sec:experiments_resnet}

\begin{table}
\setlength{\tabcolsep}{4pt}
\caption{Performance of ResNet-20 (upper table) and Wide ResNet-28-10 (lower table) models on CIFAR-10/100.
We separately compare the \textit{efficient} methods (2 rightmost columns) and we mark in bold the best results (within one standard error).  
Our two methods {\color{MidnightBlue}hyper-deep/hyper-batch ensembles} improve upon deep/batch ensembles.
}
\label{tab:wide_resnet}

\centering
\resizebox{0.92\textwidth}{!}{%
\begin{tabular}{@{}clccc||cc@{}}
\toprule
 (ResNet-20) &  & single~(1) & deep ens~(4) & {\color{MidnightBlue}hyper-deep ens~(4)}& batch ens~(4) & {\color{MidnightBlue}hyper-batch ens~(4)} \\ \midrule \midrule
 \multirow{3}{*}{\shortstack{cifar100}} 
 & nll $\ \,\downarrow$ & 1.178 \spm{0.020} & 0.971 \spm{0.002} & \textbf{0.925}\spm{0.002} & 1.235  \spm{0.007} & \textbf{1.152} \spm{0.015} \\
 & acc $\,\uparrow$  & 0.682 \spm{0.005} & \textbf{0.726} \spm{0.000} & \textbf{0.742}\spm{0.001} & 0.697 \spm{0.000} & \textbf{0.699} \spm{0.002} \\
 & ece $\,\downarrow$ & 0.064  \spm{0.005}  &  0.059 \spm{0.000} & \textbf{0.049} \spm{0.001} & 0.119  \spm{0.001} & \textbf{0.095} \spm{0.002} \\
 & div $\,\uparrow$ &  --  & \textbf{1.177} \spm{0.004} & \textbf{1.323} \spm{0.001} & \textbf{0.154}\spm{0.006} & \textbf{0.159} \spm{0.007} \\
 \cmidrule(l){1-7}
 \multirow{3}{*}{\shortstack{cifar10}} 
 & nll $\ \,\downarrow$ & 0.262  \spm{0.006} & \textbf{0.193} \spm{0.000} & \textbf{0.192}\spm{0.004} & 0.278 \spm{0.004} & \textbf{0.235} \spm{0.004} \\
 & acc $\,\uparrow$  & 0.927   \spm{0.001} & 0.937  \spm{0.000} & \textbf{0.940}\spm{0.000} & \textbf{0.929}\spm{0.000} & \textbf{0.929} \spm{0.001} \\
 & ece $\,\downarrow$ & 0.035 \spm{0.001}  &  \textbf{0.010} \spm{0.000} & 0.012\spm{0.001} & 0.039\spm{0.001} & \textbf{0.017}\spm{0.000} \\
 & div $\,\uparrow$ &  --  & 1.393 \spm{0.025} & \textbf{1.451} \spm{0.018} & 0.789 \spm{0.010} & \textbf{0.821} \spm{0.013} \\
 \bottomrule
\end{tabular}%
 }\\~\\
\centering
\resizebox{0.92\textwidth}{!}{%
\begin{tabular}{@{}clccc||cc@{}}
\toprule
 (WRN-28-10) &  & single~(1) & deep ens~(4) & {\color{MidnightBlue}hyper-deep ens~(4)}& batch ens~(4) & {\color{MidnightBlue}hyper-batch ens~(4)} \\ \midrule \midrule
 \multirow{3}{*}{\shortstack{cifar100}} 
 & nll $\ \,\downarrow$ &  0.811 \spm{0.026} & 0.661 \spm{0.001} & \textbf{0.652}\spm{0.000} & 0.690 \spm{0.005} & \textbf{0.678} \spm{0.005} \\
 & acc $\,\uparrow$  &   0.801 \spm{0.004} &  0.826 \spm{0.001} & \textbf{0.828}\spm{0.000} & \textbf{0.819} \spm{0.001} & \textbf{0.820} \spm{0.000} \\
 & ece $\,\downarrow$ &  0.062 \spm{0.001}  &  0.022 \spm{0.000} & \textbf{0.019} \spm{0.000} & 0.026 \spm{0.002} & \textbf{0.022} \spm{0.001} \\
 & div $\,\uparrow$ &  --  &  0.956\spm{0.009} & \textbf{1.086} \spm{0.011} & 0.761\spm{0.014} & \textbf{0.996} \spm{0.015} \\
 \cmidrule(l){1-7}
 \multirow{3}{*}{\shortstack{cifar10}} 
 & nll $\ \,\downarrow$ &  0.152 \spm{0.009} & 0.125 \spm{0.000} & \textbf{0.115}\spm{0.001} & 0.136 \spm{0.001} & \textbf{0.126} \spm{0.001} \\
 & acc $\,\uparrow$  &   0.961 \spm{0.001} &  0.962 \spm{0.000} & \textbf{0.965}\spm{0.000} & \textbf{0.963}\spm{0.001} & \textbf{0.963} \spm{0.000} \\
 & ece $\,\downarrow$ &  0.023\spm{0.005}  &  \textbf{0.007}\spm{0.000} & \textbf{0.007}\spm{0.000} & 0.017\spm{0.001} & \textbf{0.009}\spm{0.001} \\
 & div $\,\uparrow$ &  --  &  0.866 \spm{0.017} & \textbf{1.069} \spm{0.025} & 0.444\spm{0.003} & \textbf{0.874} \spm{0.026} \\
 \bottomrule
\end{tabular}%
 }
 
\end{table}

We evaluate our approach in a large-scale setting with ResNet-20~\cite{he2016deep} and Wide ResNet 28-10 models~\cite{zagoruyko2016wide} as they are simple architectures with competitive performance on image classification tasks. We consider six different $L_2$ regularization hyperparameters (one for each block of the ResNet) and a label smoothing hyperparameter. We show results on CIFAR-10, CIFAR-100 and corruptions on CIFAR-10~\cite{hendrycks2019benchmarking, snoek2019can}. Moreover, in~\Cref{sec:supp_results_full_odd}, we provide additional out-of-distribution evaluations along the line of~\cite{hein2019relu}. Further details about the experiment settings can be found in \Cref{sec:supp_resnet_details}.

\paragraph{CIFAR-10/100.} We compare \strathyperens{} with a \single{} model (tuned as next explained) and \deepens{} of varying ensemble sizes. Our \strathyperens{} is constructed based on 100 trials of random search while \deepens{} and \single{} take the best hyperparameter configuration found by the random search procedure. \Cref{fig:str_hyper_ens_cifar100} displays the results on CIFAR-100 along with the standard errors and shows that throughout the ensemble sizes, there is a substantial performance improvement of hyper-deep ensembles over deep ensembles. The results for CIFAR-10 are shown in \Cref{sec:supp_resnet_details} where \strathyperens{}
leads to consistent but smaller improvements, e.g., in terms of NLL.
We next fix the ensemble size to four and compare the performance of \batchhyperens{} with the direct competing method \batchens{}, as well as with \strathyperens{}, \deepens{} and \single{}.

\begin{wrapfigure}{r}{0.6\textwidth}
\begin{minipage}{0.6\textwidth}
\vspace{-0.5cm}
\resizebox{0.98\textwidth}{!}{
\includegraphics[scale=1.]{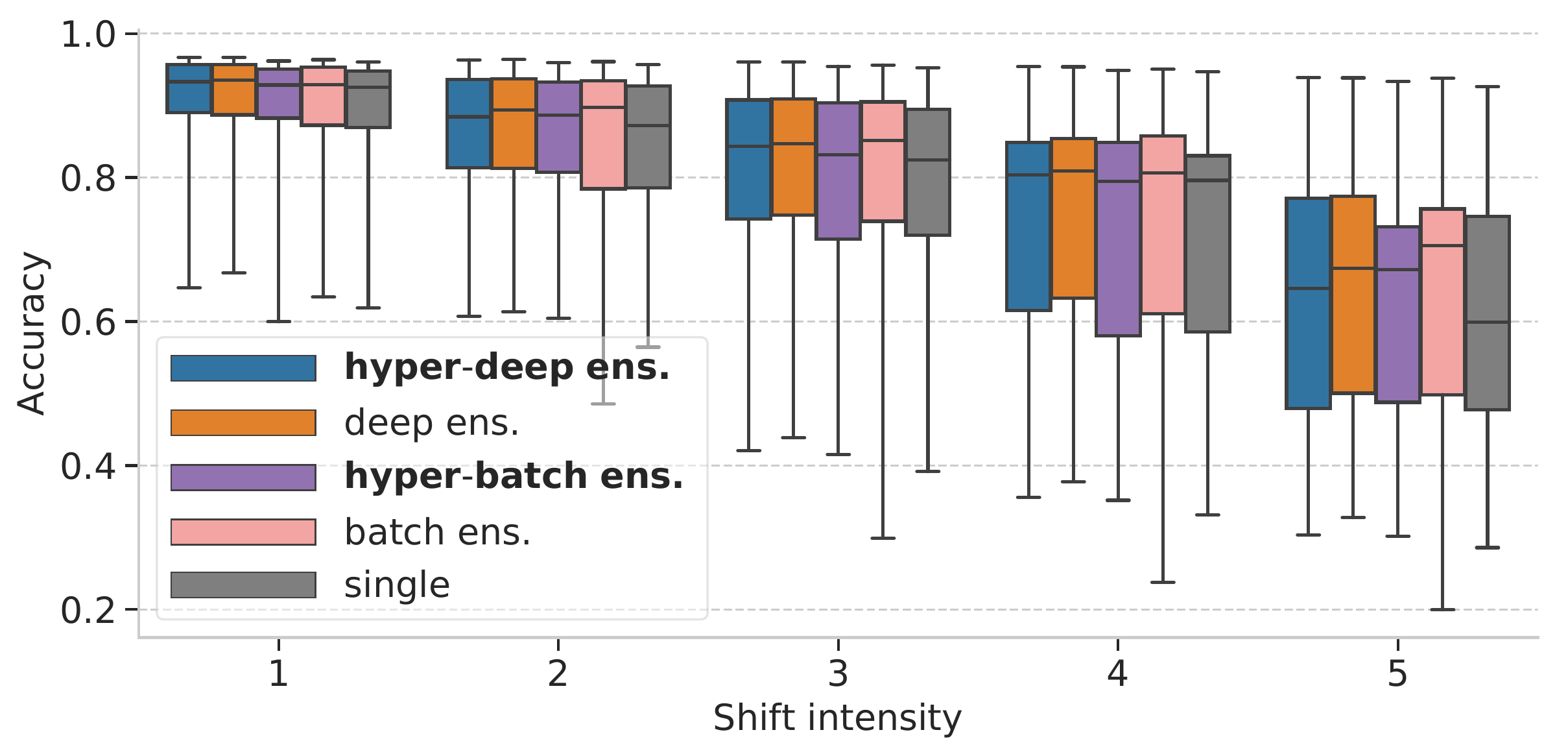}
}
\vspace{-0.2cm}%
\captionof{figure}{CIFAR-10 corruptions. Each box shows the quartiles summarizing the results
across all types of shifts while the error bars give the min/max across different shift types.}
\label{fig:cifar10_ood}
\vspace{-0.1cm}%
\end{minipage}
\end{wrapfigure}
The results are reported in \Cref{tab:wide_resnet}. On CIFAR-100, \batchhyperens{} improves, or matches, \batchens{} across all metrics.
For instance, in terms of NLL, it improves upon \batchens{} by about 7\% and 2\% for ResNet-20 and Wide ResNet 28-10 respectively. 
Moreover, the members of \batchhyperens{} make more diverse predictions than those of \batchens{}. On CIFAR-10 \batchhyperens{} also achieves a consistent improvement, though less pronounced (see~\Cref{tab:wide_resnet}). On the same Wide ResNet 28-10 benchmark, with identical training and evaluation pipelines (see \url{https://github.com/google/uncertainty-baselines}), variational inference~\cite{wen2018flipout} leads to (NLL, ACC, ECE)=(0.211, 0.947, 0.029) and (NLL, ACC, ECE)=(0.944, 0.778, 0.097) for CIFAR-10 and CIFAR-100 respectively, while Monte Carlo dropout~\cite{gal2016dropout} gets (NLL, ACC, ECE)=(0.160, 0.959, 0.024) and (NLL, ACC, ECE)=(0.830, 0.776, 0.050) for CIFAR-10 and CIFAR-100 respectively.

We can finally look at how the joint training in \batchhyperens{} leads to \textit{complementary} ensemble members. For instance, for Wide ResNet 28-10 on CIFAR-100, while the ensemble performance are (NLL, ACC)=(0.678, 0.820) (see~\Cref{tab:wide_resnet}), the {individual members} obtain substantially poorer performance, as measured by the \textit{average ensemble-member metrics}  (NLL, ACC)=(0.904, 0.788).

\paragraph{Training time and memory cost.} 
Both in terms of the number of parameters and training time, \batchhyperens{} is about twice as costly as \batchens{}. For CIFAR-100, \batchhyperens{} takes 2.16 minutes/epoch and \batchens{} 1.10 minute/epoch. More details are available in~\Cref{sec:supp_time_and_memory}.

\paragraph{Calibration on out of distribution data.} We measure the calibrated prediction on corrupted datasets, which is a type of out-of-distribution examples. We consider the recently published dataset by~\cite{hendrycks2019benchmarking}, which consists of over 30 types of corruptions to the images of CIFAR-10. A similar benchmark can be found in~\cite{snoek2019can}. On~\Cref{fig:cifar10_ood}, 
we find that all ensembles methods improve upon the single model. The mean accuracies are similar for all ensemble methods, whereas \batchhyperens{} shows more robustness than \batchens{} as it typically leads to smaller worst values (see bottom whiskers in~\Cref{fig:cifar10_ood}).
Plots for calibration error and NLL can be found in \Cref{sec:supp_results_cifar100}.

\section{Discussion}

We envision several promising directions for future research.

\paragraph{Towards more compact parametrization.} In this work, we have used the layers from~\cite{mackay2019self} that lead to a 2x increase in memory compared with standard layers. In lieu of~(\ref{eq:dense_layer_stn}), \textit{low-rank} parametrizations, e.g., $\Wb + \sum_{j=1}^h e_j(\lambdab) \gb_j \hb_j^\top$, would be appealing to reduce the memory footprint of self-tuning networks and hyper-batch ensembles. We formally show in~\Cref{sec:supp_mat_more_compact_stn_layers} that this family of parametrizations is well motivated in the case of shallow models where they enjoy good approximation guarantees.

\paragraph{Architecture diversity.}
Our proposed hyperparameter ensembles provide diversity with respect to hyperparameters related to regularization and optimization. We would like to go further in ensembling very different functions in the search space, such as network width, depth~\cite{antoran2020depth}, and the choice of residual block.
Doing so connects to older work on Bayesian marginalization over structures \cite{kemp2008discovery,adams2010learning}.
More broadly, we can wonder what other \textit{types} of diversity matter to endow deep learning models with better uncertainty estimates?

\section*{Broader Impact}

Our work belongs to a broader research effort that tries to quantify the predictive uncertainty for deep neural networks. Those models are known to generalize poorly to small changes to the data while maintaining high confidence in their predictions.

\paragraph{Who may benefit from this research?}

The broader topic of our work is becoming increasingly important in a context where machine learning systems are being deployed in safety-critical fields, e.g., medical diagnosis~\cite{miotto2016deep, liu2020deep} and self-driving cars~\cite{levinson2011towards}.
Those examples would benefit from the general technology we contribute to.
In those cases, it is essential to be able to reliably trust the uncertainty output by the models before any decision-making process, to possibly escalate uncertain decisions to appropriate human operators. 

\paragraph{Who may be put at disadvantage from this research?}

We are not aware of a group of people that may be put at disadvantage as a result of this direct research.

\paragraph{What are the consequences of failure of the system?}

By definition, our research could contribute to aspects of machine-learning systems used in high-risk domains (e.g., we mentioned earlier medical fields and self-driving cars) which involves complex data-driven decision-making processes. Depending on the nature of the application at hand, a failure of the system could lead to extremely negative consequences. A case in point is the recent screening system used by one third of UK government councils to allocate welfare budget.
\footnote{ \href{https://www.theguardian.com/society/2019/oct/15/councils-using-algorithms-make-welfare-decisions-benefits}{Link} to the corresponding article in The Guardian, October 2019:\\
{\scriptsize https://www.theguardian.com/society/2019/oct/15/councils-using-algorithms-make-welfare-decisions-benefits}.}

\paragraph{Do the task/method leverage biases in the data?}

The method we develop in this work is domain-agnostic and does not rely on specific data assumptions. Our method also does not contain components that would prevent its combination with existing fairness or privacy-preserving technologies~\cite{barocas2018fairness}.

\section*{Acknowledgments}
We would like to thank Nicolas Le Roux, Alexey Dosovitskiy and Josip Djolonga for insightful discussions at earlier stages of this project. Moreover, we would like to thank Sebastian Nowozin, Klaus-Robert M\"uller and Balaji Lakshminarayanan for helpful comments on a draft of this paper.

\bibliographystyle{abbrv}
\bibliography{MainBibliography.bib,temp_bibliography.bib}

\clearpage

\section*{Supplementary Material: \\Hyperparameter Ensembles for Robustness and Uncertainty Quantification}

\appendix
\section{Further details about fixed init hyper ensembles and hyper-deep ensembles}\label{sec:supp_mat_details_about_stratified_hyper_ens}

We recall the procedure from~\cite{Caruana2004} in Algorithm~\ref{alg:hyper_ens}. In words, given a pre-defined set of models $\Mcal$ (e.g., the outcome of random search), we greedily grow an ensemble, until some target size $K$ is met, by selecting with replacement the model leading to the best improvement of some score $\Scal$ such as the validation negative log-likelihood.

The \textit{with-replacement} selection strategy makes it possible to construct ensembles where the contributions of each member is weighted (see Section 2.1 in~\cite{Caruana2004}). To properly account for the fact that there may be multiple times the same model selected, we use ``.unique()'' in Algorithms~\ref{alg:stratified_hyper_ens}-\ref{alg:hyper_ens} to correctly count the number of members.

\begin{algorithm2e}[H]
\SetAlgoLined
 ensemble $\Ecal = \{\ \}$, score\  $\Scal(\cdot)$,  $\Scal_\text{best}=+\infty$\;
 \While{$|\Ecal.\mathrm{{unique}()}| \leq K$}{
  $
  f_{\thetab^\star} = \argmin_{f_\thetab \in \Mcal}\Scal(\Ecal \cup \{f_\thetab\})
  $\;
  \eIf{$\Scal(\Ecal \cup \{f_{\thetab^\star}\}) < \Scal_\mathrm{best}$}{
   $\Ecal = \Ecal \cup \{f_{\thetab^\star}\},\ \Scal_\text{best} = \Scal(\Ecal)$\;
   }{
   \textbf{return} $\Ecal$\;
  }
 } \textbf{return} $\Ecal$\;
 \caption{\texttt{hyper\_ens}($\Mcal,\ K$) \#\ \ Caruana et al.~\cite{Caruana2004}}
 \label{alg:hyper_ens}
\end{algorithm2e}

\section{Further details about hyper-batch ensemble}

\subsection{The structure of the convolutional layer}\label{sec:supp_mat_conv2d_batch_hyperparameter_ensemble}

We detail the structure of the (two-dimensional) convolutional layer of hyper-batch ensemble in the case of $K$ ensemble members.
Similar to the dense layer presented in~\Cref{sec:batch_hyperparameter_ensemble_layer}, the convolutional layer is obtained by \textit{composing} the layer of batch ensemble~\cite{wen2020batchensemble} and that of self-tuning networks~\cite{mackay2019self}.

Let us denote by $\Kb \in \Real^{l \times l \times c_\text{in} \times c_\text{out}}$ and $\bb_k \in \Real^{c_\text{out}}$ the convolution kernel and the $k$-th member-specific bias term, with $l$ the kernel size, $c_\text{in}$ the number of input channels and $c_\text{out}$ the number of output channels (also referred to as the number of filters).

For $k \in \{1,\dots, K\}$, let us consider the following auxiliary vectors $\rb_k, \ub_k \in \Real^{c_\text{in}}$ and $\sbb_k, \vb_k \in \Real^{c_\text{out}}$.
For $\Deltab$ of the same shape as $\Kb$ and the embedding $\eb(\lambdab_k) \in \Real^{c_\text{out}}$, we have
\begin{equation}\label{eq:supp_mat_conv2d_kernel_batch_hyper_ens}
\Kb_k(\lambdab_k) = 
\Kb \circ (\rb_k \sbb_k^\top) + 
[\Deltab \circ (\ub_k \vb_k^\top)] \circ \eb(\lambdab_k)^\top
\end{equation}
where the rank-1 factors are understood to be broadcast along the first two dimensions.
Similar, for the bias terms, we have
\begin{equation}\label{eq:supp_mat_conv2d_bias_batch_hyper_ens}
\bb_k(\lambdab_k) = 
\bb_k + 
\deltab_k \circ \eb'(\lambdab_k)
\end{equation}
with $\deltab_k, \eb'(\lambdab_k)$ of the same shape as $\bb_k$.

Given the form of~(\ref{eq:supp_mat_conv2d_kernel_batch_hyper_ens}) and~(\ref{eq:supp_mat_conv2d_bias_batch_hyper_ens}), we can observe that the conclusions drawn for the dense layer in~\Cref{sec:batch_hyperparameter_ensemble_layer} also hold for the convolutional layer.

\subsection{Efficient computation of the $L_2$ regularizer}\label{sec:supp_mat_l2_reg_batch_hyperparameter_ensemble}

We recall that each ensemble member manipulates its own hyperparameters $\lambdab_k \in \Real^m$ and, as required by the training procedure in~(\ref{eq:batch_hyperparameter_ensemble_training}), those hyperparameters are sampled as part of the stochastic optimization.

We focus on the example of a given dense layer, with weight matrix $\Wb_k(\lambdab_k)$ and bias term $\bb_k(\lambdab_k)$, as exposed in~\Cref{sec:batch_hyperparameter_ensemble_layer}.

Let us consider a minibatch of size $b$ for the $K$ ensemble members, i.e., $\{\lambdab_{k, i}\}_{k=1}^{K}$ for $i \in \{1,\dots, b\}$.
Moreover, let us introduce the scalar $\nu_{k, i}$ that is equal to the entry in $\lambdab_{k, i}$ containing the value of the $L_2$ penalty for the particular dense layer under study.\footnote{The precise relationship between $\nu_{k,i}$ and $\lambdab_{k,i}$ depends on the implementation details and on how the hyperparameters of the problem, e.g., the dropout rates or $L_2$ penalties, are stored in the vector $\lambdab_{k,i}$.}

With that notation, we concentrate on the efficient computation (especially the vectorization with respect to the minibatch dimension) of
\begin{equation}\label{eq:supp_mat_full_l2_reg}
\frac{1}{b K}
\sum_{i=1}^b
\sum_{k=1}^K
\nu_{k,i}
\| \Wb_k(\lambdab_{k,i}) \|^2,
\end{equation}
the case of the bias term following along the same lines.
From~\Cref{sec:batch_hyperparameter_ensemble_layer} we have
\begin{equation*}%
\Wb_k(\lambdab_{k,i}) =
 \Wb \circ (\rb_k \sbb_k^\top) + 
[\Deltab \circ (\ub_k \vb_k^\top)] \circ \eb(\lambdab_{k,i})^\top =
\Wb_k + \Deltab_k \circ \eb_{k,i}^\top
\end{equation*}
which we have simplified by introducing a few additional shorthands. Let us further introduce
\begin{equation*}
\langle \nu_k \rangle = \frac{1}{b} \sum_{i=1}^b \nu_{k, i}
\ \text{and}\
\langle \nu_k \eb_k \rangle = \frac{1}{b} \sum_{i=1}^b \nu_{k, i} \eb_{k, i}
\ \text{and}\
\langle \nu_k \eb_k^2 \rangle = \frac{1}{b} \sum_{i=1}^b \nu_{k, i} ( \eb_{k, i} \circ \eb_{k, i}).
\end{equation*}
We then develop $\|\Wb_k(\lambdab_{k,i})\|^2$ into $\|\Wb_k\|^2 + 2 \Wb_k^\top (\Deltab_k \circ \eb_{k, i}^\top) + \|\Deltab_k \circ \eb_{k, i}^\top\|^2$ and plug the decomposition into~(\ref{eq:supp_mat_full_l2_reg}), with $\Deltab_k^2 = \Deltab_k \circ \Deltab_k$, leading to
\begin{equation*}
\frac{1}{K}
\sum_{k=1}^K
\Big\{
\langle \nu_k \rangle
\|\Wb_k\|^2
+
2 \Wb_k^\top (\Deltab_k \circ \langle \nu_k \eb_k \rangle^\top)
+
\sum_{l,l'} (\Deltab_k^2)_{l,l'} \langle \nu_k \eb_k^2 \rangle_{l'}
\Big\}
\end{equation*}
for which all the remaining operations can be efficiently broadcast. 

\subsection{Details about the choice of the distributions $p_t$}\label{sec:supp_mat_choice_p_t_batch_hyperparameter_ensemble}

We discuss in this section additional details about the choice of the distributions over the hyperparameters $p_t(\lambdab_k)=p(\lambda_k|\xib_{k,t})$.

In the experiments of \Cref{sec:experiments}, we manipulate hyperparameters $\lambdab_k$'s that are positive and bounded (e.g., a dropout rate). To simplify the exposition, let us focus momentarily on a single ensemble member ($K=1$).
Let us further consider such a positive, bounded one-dimensional hyperparameter $\lambda \in [a, b]$, with $\ 0 < a < b$, and define $\phi(t) = (b-a)\ \texttt{sigmoid}(t)+a$, with $\phi^{-1}$ its inverse. In that setting, \cite{mackay2019self} propose to use for $p_t(\lambda) = p(\lambda|\xib_t)$ the following distribution:
\begin{equation}\label{eq:p_t_mackay}
\lambda | \xib_t \sim  \phi\big(\phi^{-1}(\lambda_t) + \varepsilon\big)
\ \ \text{with}\ \
\varepsilon \sim \Ncal(0, \sigma_t)
\ \ \text{and}\ \
\xib_t = \{\sigma_t, \lambda_t\}.
\end{equation}
In preliminary experiments we carried out, we encountered issues with~(\ref{eq:p_t_mackay}), e.g., $\lambda$ consistently pushed to its lower bound $a$ during the optimization.

We have therefore departed from~(\ref{eq:p_t_mackay}) and have focused instead on a simple log-uniform distribution, which is a standard choice for hyperparameter tuning, e.g.,~\cite{Bergstra2011,Bergstra2012,Mendoza2016}. Its probability density function is given by 
\begin{equation*}
p(\lambda|\xib_t) = 1/(\lambda \log(b/a))\ \text{with}\ \xib_t = \{a, b\},
\end{equation*}
while its entropy equals $\Hcal[p(\lambda|\xib_t)] = 0.5(\log(a) + \log(b)) + \log(\log(b/a))$. 
The mean of the distribution is given by $(b-a)/(\log(b)-\log(a))$ and is used to make predictions.

To summarize, and going back to the setting with $K$ ensemble members and $m$-dimensional $\lambdab_k$'s, the optimization of $\{\xib_{k,t}\}_{k=1}^K$ in the validation step involves $2mK$ parameters, i.e., the lower/upper bounds for each hyperparameter and for each ensemble member (in practice, $K\approx 5$ and $m\approx 5-10$).

\section{Further details about the MLP and LeNet experiments}

We provide in this section additional material about the experiments based on MLP and LeNet.

\subsection{MLP and LeNet archtectures and experimental settings}\label{sec:supp_mat_arch_mlp_lenet}

The architectures of the models are:
\begin{itemize}
    \item \textbf{MLP}: The multi-layer perceptron is composed of 2 hidden layers with 200 units each. The activation function is ReLU. Moreover a dropout layer is added before the last layer.
    \item \textbf{LeNet}~\cite{lecun1990handwritten}: This convolutional neural network is composed of a first conv2D layer (32 filters) with a max-pooling operation followed by a second conv2D layer (64 filters) with a max-pooling operation and finally followed by two dense layers (512 and number-of-classes units). The activation function is ReLU everywhere. Moreover, we add a dropout layer before the last dense layer.
\end{itemize}

As briefly discussed in the main paper, in the first tuning setting (i), there are two $L_2$ regularization parameters for those models: one for all the weight matrices and one for all the bias terms of the conv2D/dense layers; in the second tuning setting (ii), the $L_2$ regularization parameters are further split on a per-layer basis (i.e., a total of $3\times2=6$ and $4\times2=8$ $L_2$ regularization parameters for MLP and LeNet respectively). 

The ranges for the dropout and $L_2$ parameters are $[10^{-3}, 0.9]$ and $[10^{-3}, 10^{3}]$ across all settings (i)-(ii), models and datasets (CIFAR-100 and Fashion MNIST).

We take the official train/test splits of the two datasets, and we further subdivide (80\%/20\%) the train split into actual train/validation sets. We use everywhere Adam~\cite{kingma2013auto} with learning rate $10^{-4}$, a batchsize of 256 and 200 (resp.~500) training epochs for LeNet (resp. MLP). We tune all methods to minimize the validation NLL. All the experiments are repeated with 3 random seeds.

\subsection{Selection of the hyperparameters of batch ensemble}\label{sec:supp_mat_hyperparameters_batch_ens}

Following the recommendations from~\cite{wen2020batchensemble}, we tuned
\begin{itemize}
    \item The type of the initialization of the vectors $\rb_k$'s and $\sbb_k$'s (see \Cref{sec:background_deep_ens_and_batch_ens}). We indeed observed that the performance was sensitive to this choice. We selected from the different initialization schemes proposed in~\cite{wen2020batchensemble}
    \begin{itemize}
        \item Entries distributed according to the Gaussian distribution $\Ncal(\oneb, 0.5 \times \Ib)$
        \item Entries distributed according to the Gaussian distribution $\Ncal(\oneb, 0.75 \times \Ib)$
        \item Random independent signs, with probability of $+1$ equal to $0.5$
        \item Random independent signs, with probability of $+1$ equal to $0.75$
    \end{itemize}
    \item A scale factor $\kappa$ to make it possible to reduce the learning rate applied to the vectors $\rb_k$'s and $\sbb_k$'s. Following~\cite{wen2020batchensemble}, we considered the scale factor $\kappa$ in $\{1.0, 0.5\}$.
    \item Whether to use the Gibbs or ensemble cross-entropy at training time. Early experiments showed that Gibbs cross-entropy was substantially better so that we kept this choice fixed thereafter.
    \item Whether to regularize the vectors $\rb_k$'s and $\sbb_k$'s. \cite{wen2020batchensemble} mentioned that the two options perform equally well while we observed in those smaller-scale experiments that batch ensemble could overfit in absence of regularization. 
\end{itemize}

The two batch ensemble-specific hyperparameters above (initialization type and $\kappa$) together with the MLP/LeNet hyperparameters were tuned by 50 trials of random search, separately for each ensemble size (3 and 5) and for each triplet (dataset, model type, tuning setting).

\subsection{Selection of the hyperparameters of self-tuning networks}\label{sec:supp_mat_hyperparameters_stn}

We re-used as much as possible the hyperparameters and design choices from~\cite{mackay2019self}, i.e., 5 warm-up epochs (during which no tuning happens) before starting the alternating scheme (2 training steps followed by 1 tuning step).

For the tuning step, the batch size is taken to be the same as that of the training step (256), while the learning was set to $5\times 10^{-4}$.

We tuned the entropic regularization parameter $\tau \in \{0.01, 0.001, 0.0001\}$, separately for each triplet (dataset, model type, tuning setting), as done for all the methods compared in the benchmark. We observed that $\tau=0.001$ was often found to be the best option,  and it therefore constitutes a good default value, as reported in~\cite{mackay2019self}.

As studied in \Cref{sec:supp_mat_ablation_e}, we fix the embedding model $\eb(\cdot)$ to be an MLP with one hidden layer of 64 units and a tanh activation.

\subsection{Selection of the hyperparameters of hyper-batch ensemble}\label{sec:supp_mat_hyperparameters_batch_stn_ens}

We followed the very same protocol as that used for the standard self-tuning network (as described in \Cref{sec:supp_mat_hyperparameters_stn}).

By construction, we also inherit from the batch ensemble-specific hyperparameters (see \Cref{sec:supp_mat_hyperparameters_batch_ens}). To keep the protocol simple, we only tune the most important hyperparameter, namely the type of the initialization of the rank-1 terms (while the scale factor $\kappa$ to discount the learning rate was not considered). As for any other methods in the benchmark, $\tau$ and the initialization type were tuned separately for each triplet (dataset, model type, tuning setting).

For good default choices, we recommend to take $\tau=0.001$ and use an initialization scheme with random independent signs (with the probability of $+1$ equal to $0.75$).

\subsection{Choice of the embedding $\eb(\cdot)$}\label{sec:supp_mat_ablation_e}
We study the impact of the choice of the model that defines the embedding $\eb(\cdot)$.

In~\cite{mackay2019self}, $\eb(\cdot)$ is taken to be a simple linear transformation. In a slightly different context, the authors of~\cite{Dosovitskiy2020You} consider MLPs with one hidden layer of 128 or 256 units, depending on their applications. 

In the light of those previous choices, we compare the performance of different architectures of $\eb(\cdot)$, namely linear (i.e., 0 units) and one hidden layer of 64, 128, and 256 units.
The results are summarized in~\Cref{fig:sup_mat_ablation_e_and_tau}-(left), for different ensemble sizes (one corresponding to the standard self-tuning networks~\cite{mackay2019self}). We computed the validation NLL averaged over all the datasets (Fashion MNIST/CIFAR 100), model types (MLP/LeNet), tuning settings and random seeds.

Based on~\Cref{fig:sup_mat_ablation_e_and_tau}-(left), we select for $\eb(\cdot)$ an MLP with a single hidden layer of 64 units and a tanh activation function. 

\begin{figure}[t]
\vspace{-0.0cm}
\centering
\resizebox{0.45\textwidth}{!}{
\includegraphics[scale=0.5]{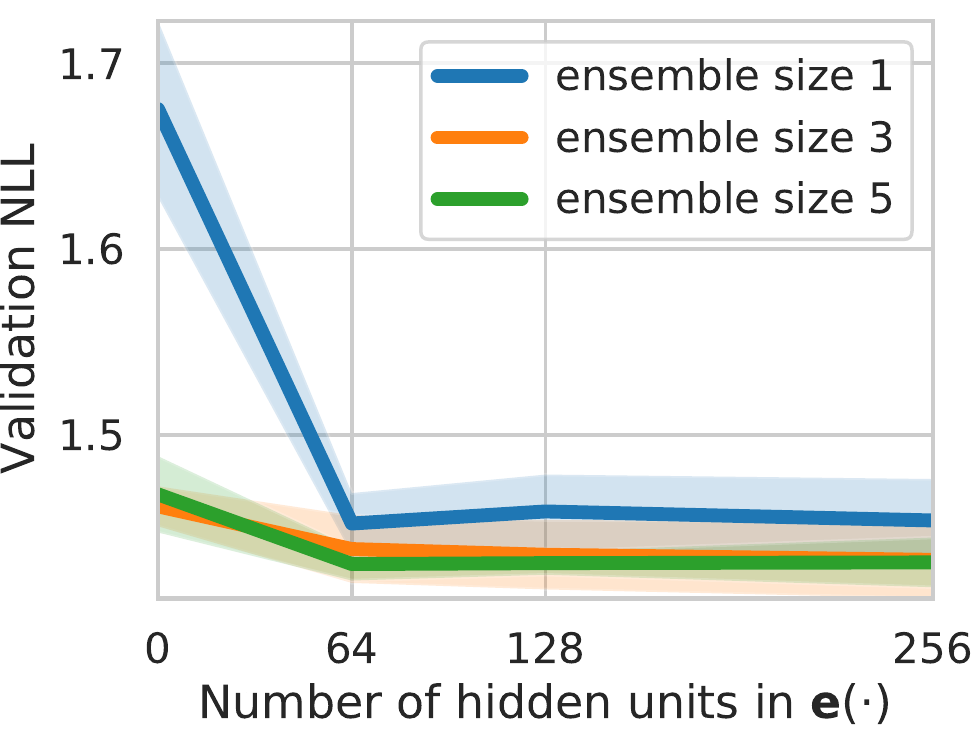}
}
\resizebox{0.45\textwidth}{!}{
\includegraphics[scale=0.5]{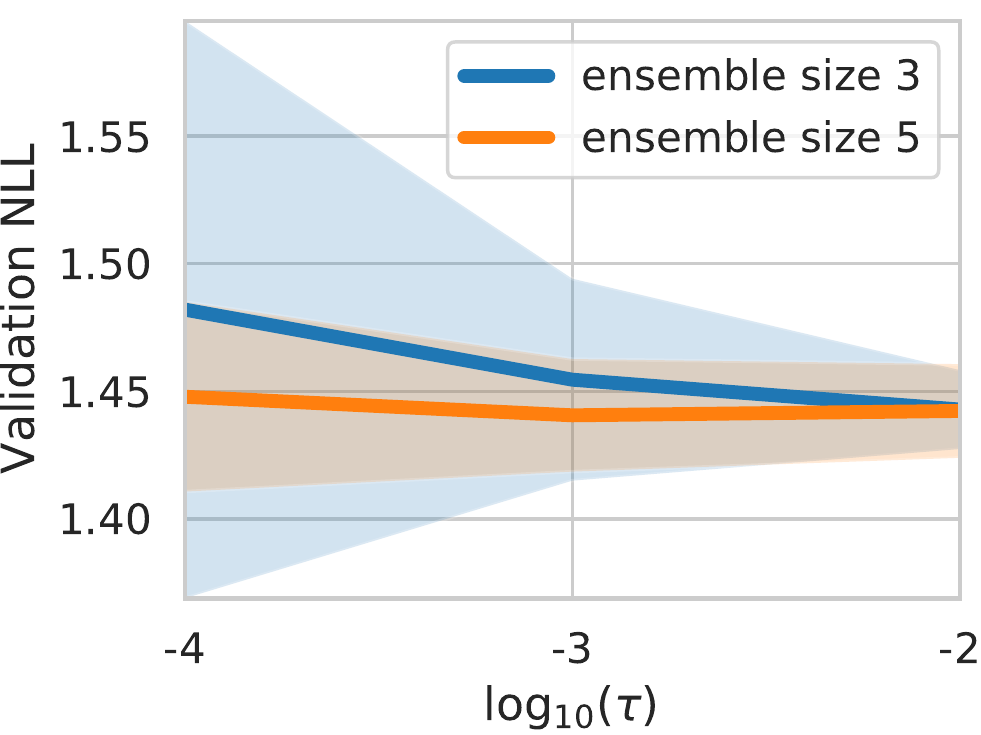}
}
\vspace{-0.0cm}%
\caption{\textsc{left:} Evolution of the validation NLL for different choices of the embedding model $\eb(\cdot)$. The validation NLL is averaged over all the datasets (Fashion MNIST/CIFAR 100), model types (MLP/LeNet), tuning settings and random seeds. Zero unit means a linear transformation without hidden layer, while $\{64, 128, 256\}$ units are for a single hidden layer. \textsc{right:} Evolution of the validation NLL for different values of $\tau$. The validation NLL is averaged over all the datasets (Fashion MNIST/CIFAR 100), model types (MLP/LeNet), tuning settings and random seeds.
}%
\label{fig:sup_mat_ablation_e_and_tau}%
\vspace{-0.0cm}%
\end{figure}

\subsection{Sensitivity analysis with respect to the entropy regularization parameter $\tau$}\label{sec:supp_mat_ablation_tau}

We study the impact of the choice of the entropy regularization parameter $\tau$ in~(\ref{eq:batch_hyperparameter_ensemble_validation}).
We report in~\Cref{fig:sup_mat_ablation_e_and_tau}-(right) how the validation negative log-likelihood---aggregated over all the datasets (Fashion MNIST/CIFAR 100), model types (MLP/LeNet), tuning settings and random seeds---varies with $\tau \in \{0.01, 0.001, 0.0001\}$.

As discussed in~\Cref{sec:supp_mat_hyperparameters_stn} and in~\Cref{sec:supp_mat_hyperparameters_batch_ens}, a good default value, as already reported in~\cite{mackay2019self} is $\tau=0.001$.

\subsection{Complementary results}

\subsubsection{Results for ensembles of size 3 and 5}\label{sec:supp_mat_ens_size_3_and_5}

In~\Cref{tab:mlp_lenet_results_ens_size_3_and_5_non_efficient_methods} and~\Cref{tab:mlp_lenet_results_ens_size_3_and_5_efficient_methods} (the latter table contains the efficient ensemble methods), we complete~\Cref{tab:mlp_lenet_results_ens_size_3} with the addition of the results for the ensembles of size 5. To ease the comparison across different ensemble sizes, we incorporate as well the results for the size 3. 

The conclusions highlighted in the main paper also hold for the larger ensembles of size 5. In~\Cref{tab:mlp_lenet_results_ens_size_3_and_5_efficient_methods}, we can observe that \batchhyperens{} with 5 members does not consistently improve upon its counterpart with 3 members. This trend is corrected if more training epochs are considered (see in~\Cref{tab:mlp_lenet_results_ens_size_3_and_5_with_ablation} the effect of twice as many training epochs).
\setlength{\tabcolsep}{4pt}
\begin{table}[t]
\caption{Comparison over CIFAR 100 and Fashion MNIST with MLP and LeNet architectures. The table reports means $\pm$ standard errors (over the 3 random seeds and pooled over the 2 tuning settings),  for ensemble approaches with 3 and 5 members (the \textit{efficient} approaches are compared separately in~\Cref{tab:mlp_lenet_results_ens_size_3_and_5_efficient_methods}).
``fixed init ens'' is a shorthand for \hyperens{}, i.e., a ``row'' in~\Cref{fig:initialization_vs_lambdas_and_example_trajectory}-(left).
Our method {\color{MidnightBlue}hyper-deep ensemble} improves upon deep ensemble (in~\Cref{sec:supp_mat_statistical_significance_mlp_lenet}, we assess the statistical significance of those improvements with a Wilcoxon signed-rank test, paired along settings, datasets and model types).}
\label{tab:mlp_lenet_results_ens_size_3_and_5_non_efficient_methods}
\centering
\resizebox{\textwidth}{!}{%
\begin{tabular}{@{}clcccccc@{}}
\toprule
 &  & fixed init ens~(3) & fixed init ens~(5) & {\color{MidnightBlue}hyper-deep ens~(3)} & {\color{MidnightBlue}hyper-deep ens~(5)} & deep ens~(3) & deep ens~(5) \\ \midrule \midrule
\multirow{4}{*}{\shortstack{cifar100 \\(mlp)}}
& nll $\, \downarrow$ & \textbf{2.943}\spm{0.010} & \textbf{2.920}\spm{0.007} & \textbf{2.953}\spm{0.058} & \textbf{2.919}\spm{0.041} & \textbf{2.969}\spm{0.057} & \textbf{2.946}\spm{0.041} \\
& acc $\, \uparrow$ & 0.287\spm{0.003} & \textbf{0.292}\spm{0.002} & \textbf{0.291}\spm{0.004} & \textbf{0.296}\spm{0.003} & 0.289\spm{0.003} & \textbf{0.292}\spm{0.004} \\
& brier $\, \downarrow$ & -0.161\spm{0.002} & -0.165\spm{0.001} & \textbf{-0.164}\spm{0.003} & \textbf{-0.169}\spm{0.002} & -0.160\spm{0.004} & -0.163\spm{0.003} \\
& ece $\,\downarrow$ & \textbf{0.029}\spm{0.007} & \textbf{0.025}\spm{0.006} & \textbf{0.022}\spm{0.007} & \textbf{0.023}\spm{0.005} & \textbf{0.038}\spm{0.014} & \textbf{0.035}\spm{0.007} \\ 
 \cmidrule(l){1-8}
\multirow{4}{*}{\shortstack{cifar100 \\(lenet)}}
& nll $\, \downarrow$ & \textbf{2.259}\spm{0.067} & \textbf{2.248}\spm{0.069} & \textbf{2.211}\spm{0.066} & \textbf{2.136}\spm{0.057} & \textbf{2.334}\spm{0.141} & \textbf{2.298}\spm{0.146} \\
& acc $\, \uparrow$ & 0.439\spm{0.008} & 0.445\spm{0.010} & 0.452\spm{0.007} & \textbf{0.466}\spm{0.006} & 0.421\spm{0.026} & 0.428\spm{0.027} \\
& brier $\, \downarrow$ & -0.301\spm{0.010} & \textbf{-0.305}\spm{0.012} & \textbf{-0.315}\spm{0.010} & \textbf{-0.330}\spm{0.008} & -0.282\spm{0.030} & -0.288\spm{0.031} \\
& ece $\,\downarrow$ & \textbf{0.049}\spm{0.023} & \textbf{0.045}\spm{0.021} & \textbf{0.039}\spm{0.013} & \textbf{0.034}\spm{0.008} & \textbf{0.050}\spm{0.015} & \textbf{0.045}\spm{0.022} \\ 
 \cmidrule(l){1-8}
\multirow{4}{*}{\shortstack{fmnist \\(mlp)}}
& nll $\, \downarrow$ & 0.312\spm{0.003} & \textbf{0.305}\spm{0.003} & 0.310\spm{0.001} & \textbf{0.305}\spm{0.001} & 0.319\spm{0.005} & 0.318\spm{0.006} \\
& acc $\, \uparrow$ & 0.893\spm{0.001} & \textbf{0.897}\spm{0.000} & 0.895\spm{0.001} & \textbf{0.897}\spm{0.000} & 0.889\spm{0.003} & 0.889\spm{0.003} \\
& brier $\, \downarrow$ & -0.843\spm{0.001} & \textbf{-0.848}\spm{0.001} & -0.845\spm{0.001} & \textbf{-0.848}\spm{0.001} & -0.839\spm{0.003} & -0.840\spm{0.003} \\
& ece $\,\downarrow$ & \textbf{0.012}\spm{0.005} & \textbf{0.014}\spm{0.002} & \textbf{0.014}\spm{0.003} & 0.017\spm{0.001} & \textbf{0.010}\spm{0.003} & \textbf{0.009}\spm{0.003} \\
 \cmidrule(l){1-8}
\multirow{4}{*}{\shortstack{fmnist \\(lenet)}}
& nll $\, \downarrow$ & 0.219\spm{0.002} & 0.215\spm{0.002} & 0.216\spm{0.002} & \textbf{0.210}\spm{0.002} & 0.226\spm{0.004} & 0.222\spm{0.005} \\
& acc $\, \uparrow$ & 0.924\spm{0.001} & \textbf{0.926}\spm{0.001} & \textbf{0.926}\spm{0.002} & \textbf{0.928}\spm{0.001} & 0.920\spm{0.002} & 0.921\spm{0.002} \\
& brier $\, \downarrow$ & -0.889\spm{0.001} & \textbf{-0.891}\spm{0.001} & \textbf{-0.890}\spm{0.002} & \textbf{-0.893}\spm{0.001} & -0.883\spm{0.003} & -0.884\spm{0.003} \\
& ece $\,\downarrow$ & \textbf{0.014}\spm{0.004} & \textbf{0.015}\spm{0.002} & 0.018\spm{0.002} & \textbf{0.014}\spm{0.003} & \textbf{0.013}\spm{0.004} & \textbf{0.011}\spm{0.003} \\
 \bottomrule
\end{tabular}%
}
\end{table}
\setlength{\tabcolsep}{4pt}
\begin{table}[t]
\caption{Comparison of the \textit{efficient} ensemble methods over CIFAR 100 and Fashion MNIST with MLP and LeNet architectures. The table reports means $\pm$ standard errors (over the 3 random seeds and pooled over the 2 tuning settings), for ensemble approaches with 3 and 5 members. Our method {\color{MidnightBlue}hyper-batch ensemble} improves upon batch ensemble (in~\Cref{sec:supp_mat_statistical_significance_mlp_lenet}, we assess the statistical significance of those improvements with a Wilcoxon signed-rank test, paired along settings, datasets and model types).}
\label{tab:mlp_lenet_results_ens_size_3_and_5_efficient_methods}
\centering
\resizebox{0.75\textwidth}{!}{%
\begin{tabular}{@{}clcccc@{}}
\toprule
 &  & {\color{MidnightBlue}hyper-batch ens~(3)} & {\color{MidnightBlue}hyper-batch ens~(5)} & batch ens~(3) & batch ens~(5) \\ \midrule \midrule
\multirow{4}{*}{\shortstack{cifar100 \\(mlp)}}
& nll $\, \downarrow$ & \textbf{2.979}\spm{0.004} & \textbf{2.983}\spm{0.001} & 3.015\spm{0.003} & 3.056\spm{0.004} \\
& acc $\, \uparrow$ & \textbf{0.281}\spm{0.002} & \textbf{0.282}\spm{0.001} & 0.275\spm{0.001} & 0.265\spm{0.001} \\
& brier $\, \downarrow$ & \textbf{-0.157}\spm{0.000} & \textbf{-0.157}\spm{0.000} & -0.153\spm{0.001} & -0.141\spm{0.000} \\
& ece $\,\downarrow$ & 0.030\spm{0.002} & 0.034\spm{0.001} & \textbf{0.022}\spm{0.002} & 0.033\spm{0.002} \\ 
 \cmidrule(l){1-6}
\multirow{4}{*}{\shortstack{cifar100 \\(lenet)}}
& nll $\, \downarrow$ & \textbf{2.283}\spm{0.016} & \textbf{2.297}\spm{0.009} & \textbf{2.350}\spm{0.024} & \textbf{2.239}\spm{0.027} \\
& acc $\, \uparrow$ & 0.428\spm{0.003} & 0.425\spm{0.002} & \textbf{0.438}\spm{0.003} & \textbf{0.437}\spm{0.006} \\
& brier $\, \downarrow$ & \textbf{-0.288}\spm{0.003} & -0.282\spm{0.002} & \textbf{-0.295}\spm{0.003} & \textbf{-0.296}\spm{0.008} \\
& ece $\,\downarrow$ & \textbf{0.058}\spm{0.004} & 0.069\spm{0.006} & \textbf{0.058}\spm{0.015} & \textbf{0.038}\spm{0.018} \\
 \cmidrule(l){1-6}
\multirow{4}{*}{\shortstack{fmnist \\(mlp)}}
& nll $\, \downarrow$ & 0.308\spm{0.002} & \textbf{0.304}\spm{0.001} & 0.351\spm{0.004} & 0.320\spm{0.002} \\
& acc $\, \uparrow$ & \textbf{0.892}\spm{0.001} & \textbf{0.892}\spm{0.001} & 0.884\spm{0.001} & \textbf{0.892}\spm{0.000} \\
& brier $\, \downarrow$ & \textbf{-0.844}\spm{0.001} & \textbf{-0.845}\spm{0.001} & -0.830\spm{0.001} & \textbf{-0.844}\spm{0.001} \\
& ece $\,\downarrow$ & 0.016\spm{0.001} & \textbf{0.013}\spm{0.001} & 0.020\spm{0.001} & 0.024\spm{0.001} \\ 
 \cmidrule(l){1-6}
\multirow{4}{*}{\shortstack{fmnist \\(lenet)}}
& nll $\, \downarrow$ & \textbf{0.212}\spm{0.001} & \textbf{0.209}\spm{0.002} & 0.230\spm{0.005} & 0.221\spm{0.002} \\
& acc $\, \uparrow$ & \textbf{0.924}\spm{0.001} & \textbf{0.925}\spm{0.001} & 0.920\spm{0.001} & 0.922\spm{0.001} \\
& brier $\, \downarrow$ & \textbf{-0.889}\spm{0.001} & \textbf{-0.891}\spm{0.001} & -0.883\spm{0.001} & -0.886\spm{0.001} \\
& ece $\,\downarrow$ & \textbf{0.009}\spm{0.001} & \textbf{0.008}\spm{0.001} & 0.017\spm{0.002} & 0.015\spm{0.001} \\ 
 \bottomrule
\end{tabular}%
}
\end{table}

\subsubsection{Assessment of the statistical significance of the results}\label{sec:supp_mat_statistical_significance_mlp_lenet}

To assess the statistical significance of the improvements displayed in \Cref{tab:mlp_lenet_results_ens_size_3_without_brier}, \Cref{tab:mlp_lenet_results_ens_size_3_and_5_non_efficient_methods} and \Cref{tab:mlp_lenet_results_ens_size_3_and_5_efficient_methods}, we run the Wilcoxon signed-rank test, paired along settings, datasets and model types. We report the results in~\Cref{tab:supp_mat_statistical_significance_mlp_lenet}. The pairing of the tests is especially important for the comparisons between \deepens{}, \hyperens{} and \strathyperens{} since their respective performances are heavily conditioned on the initial random searches they build upon. 

First, we can see that \strathyperens{} significantly improves upon both \deepens{} and \hyperens{} (with larger p-values in the latter case, though). 
Second, while \batchhyperens{} significantly improves upon \texttt{STN}, \batchhyperens{} can only be shown to be better than \batchens{} in terms of likelihood (with a 5\% significance level).
Overall, we also observe that we do not have significant improvements with respect to ECE which is known to be more noisy~\cite{nixon2019measuring}.

\begin{table}[h]
\vspace*{-0.0cm}
\caption{Results of the one-sided, Wilcoxon signed-rank test, paired along settings, datasets and model types. We report the p-values corresponding to the hypothesis that our method (in {\color{MidnightBlue}blue}) has worse value than the corresponding competing methods.}
\label{tab:supp_mat_statistical_significance_mlp_lenet}
\centering
\resizebox{\textwidth}{!}{%
\begin{tabular}{@{}l|cccc|cccc@{}}
                            & ens size & p-value (nll) & p-value (acc) & p-value (ece) & ens size & p-value (nll) & p-value (acc) & p-value (ece) \\ \midrule
\deepens{} $\leftrightarrow$ \strathyperens{} & 3        &    $1.1\times  10^{-5}$           &     $2.1\times 10^{-5}$   &  0.25    & 5        &      $9.1\times   10^{-6}$      &      $1.9\times  10^{-5}$   & 0.33   \\
\hyperens{} $\leftrightarrow$ \strathyperens{} & 3        &    0.0725           &     0.0017   &     0.43   & 5        &     0.0088         &     0.0018 & 0.44\\ 
\midrule
\midrule
\batchens{} $\leftrightarrow$ \batchhyperens{} & 3        &    $6.4\times  10^{-5}$           &     0.13   &  0.31    & 5        &      0.038      &      0.22   & 0.39   \\
\texttt{STN} $\leftrightarrow$ \batchhyperens{} & 3        &   $9.1\times  10^{-6}$          &     $2.6\times  10^{-5}$   &     0.23   & 5        &     $4.5 \times  10^{-5}$         &     $1.3\times  10^{-5}$ & 0.33
\end{tabular}
}
\vspace*{-0.0cm}
\end{table}

\clearpage
\subsubsection{Diversity analysis}\label{sec:supp_mat_mlp_lenet_diversity_analysis}

In this section, we study the diversity of the predictions made by the ensemble approaches from the experiments of~\Cref{sec:experiments_mlp_lenet}.

\setlength{\tabcolsep}{4pt}
\begin{table}[t]
\caption{Normalized predictive disagreement from~\cite{fort2019deep} compared over CIFAR 100 and Fashion MNIST with MLP and LeNet architectures. Higher values mean more diversity in the ensemble predictions. The table reports means $\pm$ standard errors (over the 3 random seeds and pooled over the 2 tuning settings), for ensemble approaches with 3 and 5 members.}
\label{tab:mlp_lenet_results_diversity_deep_vs_stratified_ens}
\resizebox{\textwidth}{!}{%
\begin{tabular}{@{}ccccc||cccc@{}}
\toprule
 &  deep ens~(3) & deep ens~(5) & {\color{MidnightBlue}hyper-deep ens~(3)} & {\color{MidnightBlue}hyper-deep ens~(5)} & batch ens~(3) & batch ens~(5) & {\color{MidnightBlue}hyper-batch ens~(3)} & {\color{MidnightBlue}hyper-batch ens~(5)} \\ \midrule \midrule
\multirow{3}{*}{\shortstack{cifar100 \\(mlp)}}
\\[-0.2em]
 &  0.570\spm{0.099} & 0.573\spm{0.103} & \textbf{0.707}\spm{0.072} & \textbf{0.732}\spm{0.055} & 0.700\spm{0.003} & 0.453\spm{0.010} & 0.765\spm{0.004} & \textbf{0.841}\spm{0.004} \\[-0.4em]
\\\cmidrule(l){1-9}
\multirow{3}{*}{\shortstack{cifar100 \\(lenet)}}
\\[-0.2em]
 &  0.688\spm{0.107} & 0.695\spm{0.114} & \textbf{0.896}\spm{0.045} & \textbf{0.896}\spm{0.038} & \textbf{0.692}\spm{0.028} & 0.583\spm{0.034} & \textbf{0.716}\spm{0.005} & 0.479\spm{0.011}\\[-0.4em]
 \\ \cmidrule(l){1-9}
\multirow{3}{*}{\shortstack{fmnist \\(mlp)}}
\\[-0.2em]
 &  0.461\spm{0.063} & 0.457\spm{0.040} & 0.588\spm{0.046} & \textbf{0.702}\spm{0.057} & 0.490\spm{0.014} & \textbf{0.716}\spm{0.003} & 0.509\spm{0.009} & 0.573\spm{0.008} \\[-0.4em]
 \\ \cmidrule(l){1-9}
\multirow{3}{*}{\shortstack{fmnist \\(lenet)}}
\\[-0.2em]
 &  0.475\spm{0.057} & 0.479\spm{0.060} & \textbf{0.594}\spm{0.041} & \textbf{0.656}\spm{0.043} & 0.481\spm{0.047} & \textbf{0.647}\spm{0.015} & 0.446\spm{0.015} & 0.487\spm{0.008}
 \\[-0.4em]
 \\ \bottomrule
\end{tabular}%
}
\end{table}

To this end, we use the predictive disagreement metric from~\cite{fort2019deep}.
This metric is based on the average of the pairwise comparisons of the predictions across the ensemble members. For a given pair of members, it is zero when they are making identical predictions, and one when all their predictions differ. We also normalize the diversity
metric by the error rate (i.e., one minus the accuracy) to avoid the case where random predictions provide the best diversity.

For ensemble sizes 3 and 5, we compare in~\Cref{tab:mlp_lenet_results_diversity_deep_vs_stratified_ens} the approaches hyper-deep ensemble, deep ensemble, hyper-batch ensemble and batch ensemble with respect to this metric. We can draw the following conclusions:
\begin{itemize}
    \item \textbf{hyper-deep ensemble vs.~deep ensemble:} Compared to deep ensemble, we can observe that hyper-deep ensemble leads to significantly more diverse predictions, across all combination of (dataset, model type) and ensemble sizes. Moreover, we can also see that the diversity only slightly increases for deep ensemble going from 3 to 5 members, while it increases more markedly for hyper-deep ensemble. We hypothesise this is due to the more diverse set of models (with varied initialization and hyperparameters) that hyper-deep ensemble can tap into.
    \item \textbf{hyper-batch ensemble vs.~batch ensemble:} The first observation is that in this setting (the observation turns out to be different in the case of the Wide Resnet 28-10 experiments), batch ensemble leads to the largest diversity in predictions compared to all the other methods. Although lower compared with batch ensemble, the diversity of hyper-batch ensemble is typically higher than, or competitive with the diversity of deep ensembles. 
\end{itemize}

\clearpage
\subsubsection{Further comparison between batch ensemble and hyper-batch ensemble}\label{sec:supp_mat_deep_batch_ens}

As described in \Cref{sec:batch_hyperparameter_ensemble_layer}, the structure of the layers of \batchhyperens{} leads to a 2x increase in memory compared with standard \batchens{}.

\setlength{\tabcolsep}{3pt}
\begin{table}[t]
\caption{Comparison of batch hyperparameter ensemble and batch ensemble over CIFAR 100 and Fashion MNIST with MLP and LeNet models, while accounting for the number of parameters. The table reports means $\pm$ standard errors (over the 3 random seeds and pooled over the 2 tuning settings), for ensemble approaches with 3 and 5 members. ``2x-'' indicates the method benefited from twice as many training epochs.
The two rightmost columns correspond to the combination of two batch ensemble models with 3 and 5 members, resulting in 6 and 10 members.}
\label{tab:mlp_lenet_results_ens_size_3_and_5_with_ablation}
\resizebox{\textwidth}{!}{
\begin{tabular}{@{}clcc|cccc@{}}
\toprule
 &  & hyper-batch ens~(3) & hyper-batch ens~(5) & 2x-hyper-batch ens~(3) &  2x-hyper-batch ens~(5) &  batch ens~(3$\times$2) & batch ens~(5$\times$2) \\ \midrule \midrule
\multirow{4}{*}{\shortstack{cifar100 \\(mlp)}}
& nll $\, \downarrow$ & 2.979\spm{0.004} & 2.983\spm{0.001} & 2.974\spm{0.006} & \textbf{2.950}\spm{0.003} & 2.980\spm{0.002} & 3.031\spm{0.002} \\
& acc $\, \uparrow$ & \textbf{0.281}\spm{0.002} & \textbf{0.282}\spm{0.001} & 0.277\spm{0.003} & \textbf{0.284}\spm{0.002} & \textbf{0.282}\spm{0.001} & 0.268\spm{0.001} \\
& brier $\, \downarrow$ & -0.157\spm{0.000} & -0.157\spm{0.000} & -0.153\spm{0.001} & \textbf{-0.159}\spm{0.001} & -0.157\spm{0.000} & -0.144\spm{0.000} \\
& ece $\,\downarrow$ & \textbf{0.030}\spm{0.002} & \textbf{0.034}\spm{0.001} & \textbf{0.033}\spm{0.004} & \textbf{0.034}\spm{0.005} & \textbf{0.032}\spm{0.001} & 0.040\spm{0.002} \\ 
 \cmidrule(l){1-8}
\multirow{4}{*}{\shortstack{cifar100 \\(lenet)}}
& nll $\, \downarrow$ & 2.283\spm{0.016} & 2.297\spm{0.009} & 2.255\spm{0.014} & 2.269\spm{0.006} & 2.188\spm{0.008} & \textbf{2.163}\spm{0.012} \\
& acc $\, \uparrow$ & 0.428\spm{0.003} & 0.425\spm{0.002} & 0.430\spm{0.003} & 0.428\spm{0.001} & \textbf{0.460}\spm{0.002} & 0.451\spm{0.003} \\
& brier $\, \downarrow$ & -0.288\spm{0.003} & -0.282\spm{0.002} & -0.295\spm{0.002} & -0.291\spm{0.001} & \textbf{-0.321}\spm{0.001} & -0.309\spm{0.004} \\
& ece $\,\downarrow$ & 0.058\spm{0.004} & 0.069\spm{0.006} & 0.028\spm{0.001} & 0.036\spm{0.006} & \textbf{0.017}\spm{0.004} & 0.060\spm{0.009} \\ 
 \cmidrule(l){1-8}
\multirow{4}{*}{\shortstack{fmnist \\(mlp)}}
& nll $\, \downarrow$ & 0.308\spm{0.002} & 0.304\spm{0.001} & 0.307\spm{0.001} & \textbf{0.303}\spm{0.001} & 0.333\spm{0.003} & 0.308\spm{0.001} \\
& acc $\, \uparrow$ & 0.892\spm{0.001} & 0.892\spm{0.001} & \textbf{0.893}\spm{0.001} & \textbf{0.894}\spm{0.001} & 0.887\spm{0.001} & \textbf{0.894}\spm{0.001} \\
& brier $\, \downarrow$ & -0.844\spm{0.001} & \textbf{-0.845}\spm{0.001} & \textbf{-0.845}\spm{0.001} &\textbf{ -0.847}\spm{0.001} & -0.836\spm{0.001} & \textbf{-0.847}\spm{0.000} \\
& ece $\,\downarrow$ & 0.016\spm{0.001} & \textbf{0.013}\spm{0.001} & \textbf{0.015}\spm{0.001} & \textbf{0.013}\spm{0.001} & 0.016\spm{0.001} & 0.020\spm{0.001} \\ \cmidrule(l){1-8}
\multirow{4}{*}{\shortstack{fmnist \\(lenet)}}
& nll $\, \downarrow$ & 0.212\spm{0.001} & \textbf{0.209}\spm{0.002} & \textbf{0.211}\spm{0.002} & \textbf{0.209}\spm{0.001} & 0.220\spm{0.001} & 0.213\spm{0.001} \\
& acc $\, \uparrow$ & \textbf{0.924}\spm{0.001} & \textbf{0.925}\spm{0.001} & \textbf{0.925}\spm{0.001} & \textbf{0.925}\spm{0.000} & 0.922\spm{0.000} & \textbf{0.923}\spm{0.001} \\
& brier $\, \downarrow$ & \textbf{-0.889}\spm{0.001} & \textbf{-0.891}\spm{0.001} & \textbf{-0.890}\spm{0.001} & \textbf{-0.891}\spm{0.001} & -0.887\spm{0.000} & -0.889\spm{0.001} \\
& ece $\,\downarrow$ & \textbf{0.009}\spm{0.001} & \textbf{0.008}\spm{0.001} & 0.013\spm{0.001} & 0.012\spm{0.001} & 0.013\spm{0.001} & 0.011\spm{0.001} \\ \bottomrule
\end{tabular}%
}
\end{table}

In an attempt to fairly account for this difference in memory footprints, we combine two batch ensemble models trained separately and whose total memory footprint amounts to that of \batchhyperens{}. This procedure leads to ensembles with 6 and 10 members to compare to \batchhyperens{} instantiated with 3 and 5 members respectively. 
To also normalize the training budget, \batchhyperens{} is given twice as many training epochs as each of the \batchens{} models.

\Cref{tab:mlp_lenet_results_ens_size_3_and_5_with_ablation} presents the results of that comparison. In an nutshell, \batchhyperens{} either continues to improve upon, or remain competitive with, \batchens{}, while still having the advantage of automatically tuning the hyperparameters of the underlying model (MLP or LeNet). 

\subsubsection{Ablation study about hyper-deep ensemble}\label{sec:supp_mat_stratified_hyper_ens_ablation_top_k}

In this section, we conduct two ablation studies about hyper-deep ensemble to better understand its components. 
We first focus on the effect of using the greedy algorithm of~\cite{Caruana2004} compared with the top-$K$ procedure used in~\cite{saikia2020optimized}. Second, we relate Algorithm~\ref{alg:stratified_hyper_ens} to the NES-RS procedure concurrently proposed by~\cite{zaidi2020neural}.

\paragraph{Greedy~\cite{Caruana2004} versus top-$K$ selection?}

Starting from the set of models generated by random search (according to the setting of Section~\ref{sec:experiments_mlp_lenet}), we apply both the greedy and top-$K$ selection strategies, as previously used in~\cite{saikia2020optimized}, to form ensembles of size 5.
We report the results of the evaluations of those strategies in \Cref{fig:sup_mat_ablation_greedy_vs_topk}.

We can observe that the greedy procedure outperforms the top-$K$ procedure. While the former has an objective aware of the \textit{ensemble performance}, the latter selects the models based only on their individual performance. 

\begin{figure}[t]
\vspace{-0.0cm}
\centering
\resizebox{0.49\textwidth}{!}{
\includegraphics[scale=0.5]{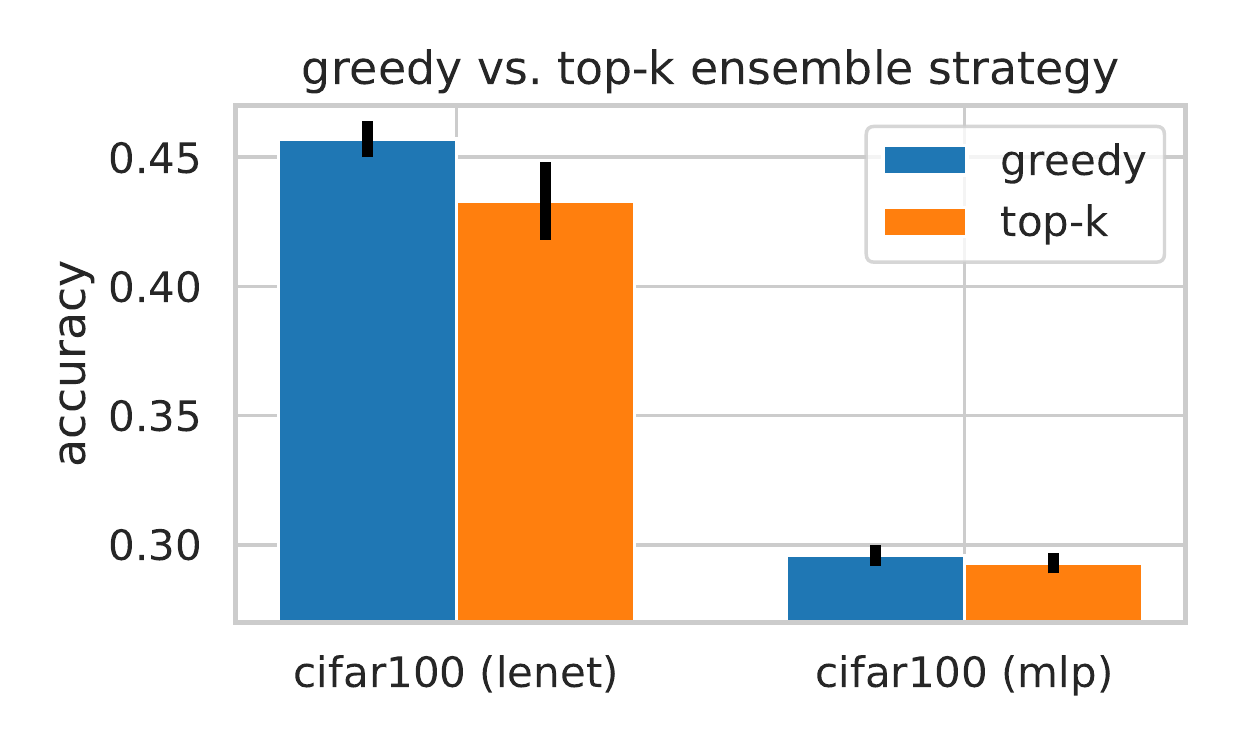}
}
\resizebox{0.49\textwidth}{!}{
\includegraphics[scale=0.5]{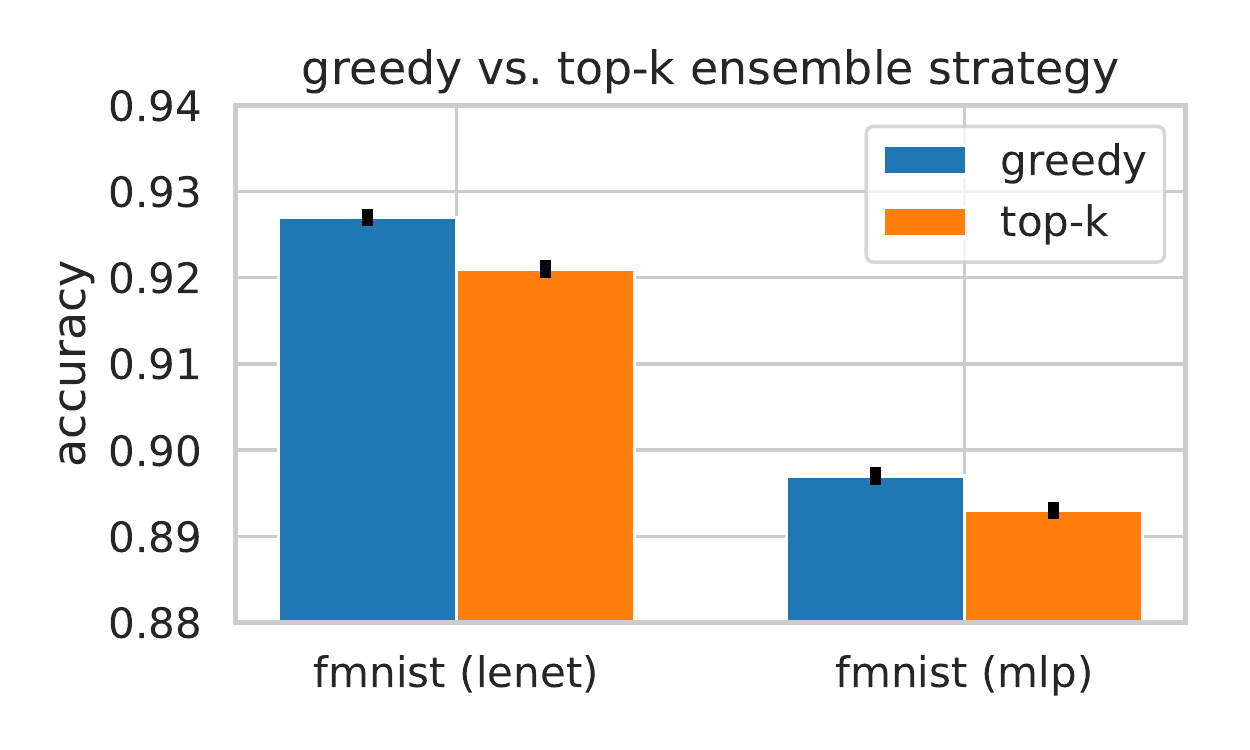}
}
\vspace{-0.0cm}%
\caption{Test accuracy evaluated over CIFAR 100 (\textsc{left}) and Fashion MNIST (\textsc{right}) for both MLP and LeNet models, when using the greedy and top-$K$ selection strategies to construct ensembles with 5 members. The accuracy is averaged over tuning settings and random seeds.}
\label{fig:sup_mat_ablation_greedy_vs_topk}%
\vspace{-0.0cm}%
\end{figure}

\paragraph{More models from random search versus fewer models with stratification?}

We still focus on the setting of Section~\ref{sec:experiments_mlp_lenet}, with ensembles of size 3 and 5. We study the value of the stratification step in Algorithm~\ref{alg:stratified_hyper_ens}. To this end, we consider the following comparison that accounts for the total number of trained models:
\begin{itemize}
    \item[\textbf{(A)}] \texttt{hyper\!~ens\!~(70)}: Random search with 70 models followed by the greedy procedure of~\cite{Caruana2004}. Note that there is no stratification step in this variant. The resulting method falls back to NES-RS from~\cite{zaidi2020neural} where the architecture is kept fixed while hyperparameters are varied.
    \item[\textbf{(B)}] \strathyperens{}: The procedure described in Algorithm~\ref{alg:stratified_hyper_ens} that uses stratification and starts from 50 models obtained by random search (as used in the experiments of Section~\ref{sec:experiments_mlp_lenet}). Note that even though we need to stratify 5 models with 5 seeds, i.e., $5^2$=25 models, we can reuse 5 models from the initial random search so that the total budget is 50+20=70 models to train (plus the cost of the calls to the greedy algorithm which is assumed negligible). The two approaches \textbf{(A)}-\textbf{(B)} therefore involve the same number of models to train.
\end{itemize}
The results of the comparison are reported in \Cref{tab:sup_mat_ablation_with_vs_without_stratification}. While \strathyperens{} works slightly better, the differences with \texttt{hyper\!~ens\!~(70)} are not substantial. In the setting of Section~\ref{sec:experiments_mlp_lenet}, it thus appears that, provided that the initial random search produces enough models, the stratification step may be bypassed. In practice, this scheme, without stratification, can also be more convenient to implement.

\begin{table}[t]
\caption{Study of the impact of the stratification when accounting for the total number of models to train. \strathyperens{} uses stratification while \texttt{hyper\!~ens\!~(70)} does not. The comparison is over CIFAR 100 and Fashion MNIST with MLP and LeNet models. The table reports means $\pm$ standard errors (over the 3 random seeds and pooled over the 2 tuning settings).}
\label{tab:sup_mat_ablation_with_vs_without_stratification}
\centering
\resizebox{\textwidth}{!}{
\begin{tabular}{lccccc}
\toprule
        &  ens size &                                                                      cifar100 (lenet) &                                                                        cifar100 (mlp) &                                                                        fmnist (lenet) &                                                                          fmnist (mlp) \\
\midrule
 hyper ens (70) &         3 &  \makecell{ce: $2.214\pm{0.054}$ \\ acc: $0.451\pm{0.006}$ \\ ece: $0.039\pm{0.009}$} &  \makecell{ce: $2.957\pm{0.047}$ \\ acc: $0.291\pm{0.002}$ \\ ece: $0.033\pm{0.008}$} &  \makecell{ce: $0.216\pm{0.003}$ \\ acc: $0.926\pm{0.001}$ \\ ece: $0.016\pm{0.003}$} &  \makecell{ce: $0.310\pm{0.003}$ \\ acc: $0.894\pm{0.001}$ \\ ece: $0.015\pm{0.002}$} \\
 & & & & &\\
 {\color{MidnightBlue}hyper-deep ens} &         3 &  \makecell{ce: $2.211\pm{0.066}$ \\ acc: $0.452\pm{0.007}$ \\ ece: $0.039\pm{0.013}$} &  \makecell{ce: $2.953\pm{0.058}$ \\ acc: $0.291\pm{0.004}$ \\ ece: $0.022\pm{0.007}$} &  \makecell{ce: $0.216\pm{0.002}$ \\ acc: $0.926\pm{0.002}$ \\ ece: $0.018\pm{0.002}$} &  \makecell{ce: $0.310\pm{0.001}$ \\ acc: $0.895\pm{0.001}$ \\ ece: $0.014\pm{0.003}$} \\
\bottomrule
\end{tabular}
}
\resizebox{\textwidth}{!}{
\begin{tabular}{lccccc}
\toprule
        &  ens size &                                                                      cifar100 (lenet) &                                                                        cifar100 (mlp) &                                                                        fmnist (lenet) &                                                                          fmnist (mlp) \\
\midrule
 hyper ens (70) &         5 &  \makecell{ce: $2.182\pm{0.053}$ \\ acc: $0.459\pm{0.005}$ \\ ece: $0.033\pm{0.005}$} &  \makecell{ce: $2.924\pm{0.035}$ \\ acc: $0.297\pm{0.002}$ \\ ece: $0.024\pm{0.004}$} &  \makecell{ce: $0.210\pm{0.001}$ \\ acc: $0.928\pm{0.001}$ \\ ece: $0.014\pm{0.002}$} &  \makecell{ce: $0.305\pm{0.002}$ \\ acc: $0.897\pm{0.001}$ \\ ece: $0.018\pm{0.004}$} \\
  & & & & &\\
 {\color{MidnightBlue}hyper-deep ens} &         5 &  \makecell{ce: $2.136\pm{0.057}$ \\ acc: $0.466\pm{0.006}$ \\ ece: $0.034\pm{0.008}$} &  \makecell{ce: $2.919\pm{0.041}$ \\ acc: $0.296\pm{0.003}$ \\ ece: $0.023\pm{0.005}$} &  \makecell{ce: $0.210\pm{0.002}$ \\ acc: $0.928\pm{0.001}$ \\ ece: $0.014\pm{0.003}$} &  \makecell{ce: $0.305\pm{0.001}$ \\ acc: $0.897\pm{0.000}$ \\ ece: $0.017\pm{0.001}$} \\
\bottomrule
\end{tabular}
}
\end{table}

\subsubsection{Addendum to the results of~\Cref{tab:mlp_lenet_results_ens_size_3_without_brier}}

In \Cref{tab:mlp_lenet_results_ens_size_3}, we complete the results of~\Cref{tab:mlp_lenet_results_ens_size_3_without_brier} with the addition of the Brier scores.
Moreover, we provide the details of the performance of \randsearch{} and \bayesopt{} since only their aggregated best results were reported in~\Cref{tab:mlp_lenet_results_ens_size_3_without_brier}.

\setlength{\tabcolsep}{4pt}
\begin{table}[t]
\caption{Comparison over CIFAR 100 and Fashion MNIST with MLP and LeNet architectures. The table reports means $\pm$ standard errors (over the 3 random seeds and pooled over the 2 tuning settings).
``fixed init ens'' is a shorthand for \hyperens{}, i.e., a ``row'' in~\Cref{fig:initialization_vs_lambdas_and_example_trajectory}-(left).
We separately compare the \textit{efficient} methods (3 rightmost columns) and we mark in bold the best results (within one standard error).  
Our two methods {\color{MidnightBlue}hyper-deep/hyper-batch ensembles} improve upon deep/batch ensembles respectively (in~\Cref{sec:supp_mat_statistical_significance_mlp_lenet}, we assess the statistical significance of those improvements with a Wilcoxon signed-rank test, paired by settings, datasets and model types).}
\label{tab:mlp_lenet_results_ens_size_3}
\resizebox{\textwidth}{!}{%
\begin{tabular}{@{}clccccc|ccc@{}}
\toprule
 &  & rand~search (1) & Bayes~opt~(1) & fixed init ens~(3) &  {\color{MidnightBlue}hyper-deep ens~(3)} & deep ens~(3) & batch ens~(3) & STN (1) &  {\color{MidnightBlue}hyper-batch ens~(3)} \\ \midrule \midrule
\multirow{4}{*}{\shortstack{cifar100 \\(mlp)}} 
& nll $\, \downarrow$ & \textbf{3.082}\spm{0.127} & 2.977\spm{0.010} & \textbf{2.943}\spm{0.010} & \textbf{2.953}\spm{0.058} & \textbf{2.969}\spm{0.057} & 3.015\spm{0.003} & 3.029\spm{0.006} & \textbf{2.979}\spm{0.004} \\
& acc $\, \uparrow$ & 0.272\spm{0.003} & 0.277\spm{0.002} & \textbf{0.287}\spm{0.003} & \textbf{0.291}\spm{0.004} & \textbf{0.289}\spm{0.003} & 0.275\spm{0.001} & 0.268\spm{0.002} & \textbf{0.281}\spm{0.002} \\
& brier $\, \downarrow$ & -0.142\spm{0.016} & -0.152\spm{0.003} & \textbf{-0.161}\spm{0.002} & \textbf{-0.164}\spm{0.003} & \textbf{-0.160}\spm{0.004} & -0.153\spm{0.001} & -0.145\spm{0.001} & \textbf{-0.157}\spm{0.000} \\
& ece $\,\downarrow$ & \textbf{0.048}\spm{0.037} & \textbf{0.034}\spm{0.008} & \textbf{0.029}\spm{0.007} & \textbf{0.022}\spm{0.007} & \textbf{0.038}\spm{0.014} & \textbf{0.022}\spm{0.002} & 0.033\spm{0.004} & 0.030\spm{0.002} \\ 
 \cmidrule(l){1-10}
\multirow{4}{*}{\shortstack{cifar100 \\(lenet)}} 
& nll $\, \downarrow$ & 2.523\spm{0.140} & \textbf{2.399}\spm{0.204} & \textbf{2.259}\spm{0.067} & \textbf{2.211}\spm{0.066} & \textbf{2.334}\spm{0.141} & 2.350\spm{0.024} & 2.329\spm{0.017} & \textbf{2.283}\spm{0.016} \\
& acc $\, \uparrow$ & 0.395\spm{0.026} & 0.420\spm{0.011} & \textbf{0.439}\spm{0.008} & \textbf{0.452}\spm{0.007} & \textbf{0.421}\spm{0.026} & \textbf{0.438}\spm{0.003} & 0.415\spm{0.003} & 0.428\spm{0.003} \\
& brier $\, \downarrow$ & -0.249\spm{0.028} & -0.270\spm{0.029} & \textbf{-0.301}\spm{0.010} & \textbf{-0.315}\spm{0.010} & \textbf{-0.282}\spm{0.030} & \textbf{-0.295}\spm{0.003} & -0.280\spm{0.002} & -0.288\spm{0.003} \\
& ece $\,\downarrow$ & \textbf{0.064}\spm{0.036} & \textbf{0.071}\spm{0.054} & \textbf{0.049}\spm{0.023} & \textbf{0.039}\spm{0.013} & \textbf{0.050}\spm{0.015} & 0.058\spm{0.015} & \textbf{0.024}\spm{0.007} & 0.058\spm{0.004} \\ 
 \cmidrule(l){1-10}
\multirow{4}{*}{\shortstack{fmnist \\(mlp)}}
& nll $\, \downarrow$ & 0.327\spm{0.005} & 0.323\spm{0.003} & \textbf{0.312}\spm{0.003} & \textbf{0.310}\spm{0.001} & 0.319\spm{0.005} & 0.351\spm{0.004} & 0.316\spm{0.003} & \textbf{0.308}\spm{0.002} \\
& acc $\, \uparrow$ & 0.888\spm{0.002} & 0.889\spm{0.002} & \textbf{0.893}\spm{0.001} & \textbf{0.895}\spm{0.001} & 0.889\spm{0.003} & 0.884\spm{0.001} & \textbf{0.890}\spm{0.001} & \textbf{0.892}\spm{0.001} \\
& brier $\, \downarrow$ & -0.836\spm{0.003} & -0.838\spm{0.002} & \textbf{-0.843}\spm{0.001} & \textbf{-0.845}\spm{0.001} & -0.839\spm{0.003} & -0.830\spm{0.001} & -0.840\spm{0.002} & \textbf{-0.844}\spm{0.001} \\
& ece $\,\downarrow$ & \textbf{0.013}\spm{0.003} & 0.022\spm{0.004} & \textbf{0.012}\spm{0.005} & \textbf{0.014}\spm{0.003} & \textbf{0.010}\spm{0.003} & 0.020\spm{0.001} & \textbf{0.016}\spm{0.001} & \textbf{0.016}\spm{0.001} \\
 \cmidrule(l){1-10}
\multirow{4}{*}{\shortstack{fmnist \\(lenet)}}
& nll $\, \downarrow$ & 0.232\spm{0.002} & 0.237\spm{0.002} & \textbf{0.219}\spm{0.002} & \textbf{0.216}\spm{0.002} & 0.226\spm{0.004} & 0.230\spm{0.005} & 0.224\spm{0.003} & \textbf{0.212}\spm{0.001} \\
& acc $\, \uparrow$ & 0.919\spm{0.001} & 0.918\spm{0.002} & \textbf{0.924}\spm{0.001} & \textbf{0.926}\spm{0.002} & 0.920\spm{0.002} & 0.920\spm{0.001} & 0.920\spm{0.001} & \textbf{0.924}\spm{0.001} \\
& brier $\, \downarrow$ & -0.881\spm{0.001} & -0.879\spm{0.002} & -0.889\spm{0.001} & -0.890\spm{0.002} & -0.883\spm{0.003} & -0.883\spm{0.001} & -0.884\spm{0.001} & -0.889\spm{0.001} \\
& ece $\,\downarrow$ & \textbf{0.019}\spm{0.004} & \textbf{0.017}\spm{0.005} & \textbf{0.014}\spm{0.004} & \textbf{0.018}\spm{0.002} & \textbf{0.013}\spm{0.004} & 0.017\spm{0.002} & 0.015\spm{0.001} & \textbf{0.009}\spm{0.001} \\
 \bottomrule
\end{tabular}%
}
\end{table}

\clearpage
\section{Further details about the ResNet experiments}
\label{sec:supp_resnet_details}

\subsection{Details about the optimization methods}
\label{sec:supp_resnet_optimizer_params}

We first explain the setting we used for training the ResNet 20 and Wide ResNet 28-10 architectures in \Cref{sec:experiments} and conclude with the results of an empirical study over different algorithmic choices.

\paragraph{Training and model definition.}
In the following we present the details for our training procedures.
A similar training setup to ours for \batchens{} based on a Wide ResNet architecture can be found in the \texttt{uncertainty-baselines} repository\footnote{\url{https://github.com/google/uncertainty-baselines/tree/master/baselines/cifar}}.

For all methods (\batchhyperens{}, \batchens{}, \strathyperens{} and \deepens{}), we optimize the model parameters using stochastic gradient descent (SGD) with Nesterov momentum of $0.9$. For the ResNet 20 model we decay the learning rate by a factor of 0.1 after the epochs $\{80, 180, 200\}$ and for the Wide ResNet 28-10 model by a factor of 0.2 after the epochs $\{100, 200, 225\}$. For tuning the hyperparameters in \batchhyperens{}, we use Adam~\cite{kingma2013auto} with a fixed learning rate.
For \batchhyperens{}, we use 95\% of the data for training and the remaining 5\% for optimizing the hyperparameters $\lambdab$ in the tuning step. For the other methods we use the full training set.

For the efficient ensemble methods (\batchhyperens{} and \batchens{}), we initialize the rank-1 factors, i.e., $\rb_k \sbb_k^\top$ and $\ub_k \vb_k^\top$ in (\ref{eq:dense_layer_batch_stn}), with entries independently sampled according to $\Ncal(1, 0.5)$ for ResNet 20 and sampled according to $\Ncal(1, 1)$ for Wide ResNet 28-10.

We make two minor adjustments of our model to adapt to the specific structure of the highly overparametrized ResNet models. First, we find that \textit{coupling} the rank-1 factors corresponding to the hyperparamters to the rank-1 factors of weights is beneficial, i.e. we set $\ub_k:=\rb_k$ and $\vb_k:=\sbb_k$. This slightly decreases the flexibility of \batchhyperens{} and makes it more robust against overfitting. 

Second, we exclude the rank-1 factors from being regularized. In the original paper introducing \batchens{} \cite{wen2020batchensemble}, the authors mention that both options were found to work equally well and they finally choose \textit{not to regularize} the rank-1 factors (to save extra computation). In our setting, we observe that this choice is important and regularizing the rank-1 factors leads to worse performance (a detailed analysis is given in \Cref{sec:supp_resnet_reg_ablation}). Hence, we do \textit{not} include the rank-1 factors in the regularization.

\setlength{\tabcolsep}{4pt}
\begin{table}[h]
\caption{Wide ResNet 28-10. Ablation for including the rank-1 factors of the efficient ensemble methods into the regularization. We run a grid search over all optimization parameters outlined in \Cref{sec:supp_resnet_optimizer_params} and report the mean performance on CIFAR-100 along with the standard error as well as the best performance attained by all configurations considered. Regularizing the factors substantially decreases the performance of both methods. The results for the unregularized version can be found in the main text, in \Cref{tab:wide_resnet}.}
\label{tab:reg_ablation}
\begin{center}
\begin{tabular}{lcccc}
\toprule
 &  Mean acc. & Max. acc. & Mean NLL & Min. NLL \\ \midrule \midrule
\batchhyperens{} & 0.797\spm{0.004} & 0.802  & 0.783\spm{0.023} & 0.750 \\
\cmidrule(l){1-5}
\batchens{} & 0.797\spm{0.004} & 0.803 & 0.782\spm{0.028} &  0.750 \\
\bottomrule
\end{tabular}%
\end{center}
\end{table}

For \batchhyperens{} we usually start with a log-uniform distribution over the hyperparameters $p_t$ over the full range for the given bounds of the hyperparameters. For the ResNet models we find that reducing the initial ranges of $p_t$ for the $L_2$ regularization parameters by one order of magnitude is more stable (but we keep the original bounds for clipping the parameters).

\paragraph{Tuning of optimization method hyperparameters.} We perform an exhaustive ablation of the different algorithmic choices for \batchhyperens{} as well as for \batchens{} using the validation set. We run a grid search procedure evaluating up to five different values for each parameter listed below and repeat each run three times using different seeds.
We find that the following configuration works best.

Shared parameters for both methods: 
\begin{itemize}
    \item The base learning rate of SGD: $0.1$.
    \item Learning rate decay ratio: $0.2$.
    \item Batch size: $64$.
    \item Initialization of each entry of the fast weights according to $\Ncal(1, 1)$.
    \item We multiply the learning rate for the fast weights by: $2.0$.
\end{itemize}

Parameters specific to hyper-batch ensemble:
\begin{itemize}
    \item Range for the $L_2$ parameters: $[0.1, 100]$.
    \item Range for the label smoothing parameter: $[0, 0.2]$.
    \item Entropy regularization parameter: $\tau=10^{-3}$ (as also used in the other experiments and used by \cite{mackay2019self}).
    \item Learning rate for the tuning step (where we use Adam): $10^{-5}$.
\end{itemize}

Remarkably, we find that the shared set of parameters which work best for batch ensembles, also work best for hyper-batch ensembles. This makes our method an easy-to-tune drop-in replacement for batch ensembles.

\subsection{Regularization of the rank-1 factors}
\label{sec:supp_resnet_reg_ablation}

As explained in the previous section, we find that for the Wide ResNet architecture, both hyper-batch ensemble and batch ensemble work best when the rank-1 factors ($\rb_k \sbb_k^\top$ and $\ub_k \vb_k^\top$) are not regularized. We examine the performance of both models when \textit{using} a regularization of the rank-1 factors. For these versions of the models, we run an ablation over the same algorithmic choices as done in the previous section. The results are displayed in \Cref{tab:reg_ablation}. The performance of both methods is substantially worse than the unregularized versions as presented in the main text, \Cref{tab:wide_resnet}.

\begin{figure}[t]
\vspace{-0.0cm}
\resizebox{\textwidth}{!}{
\includegraphics[scale=0.75]{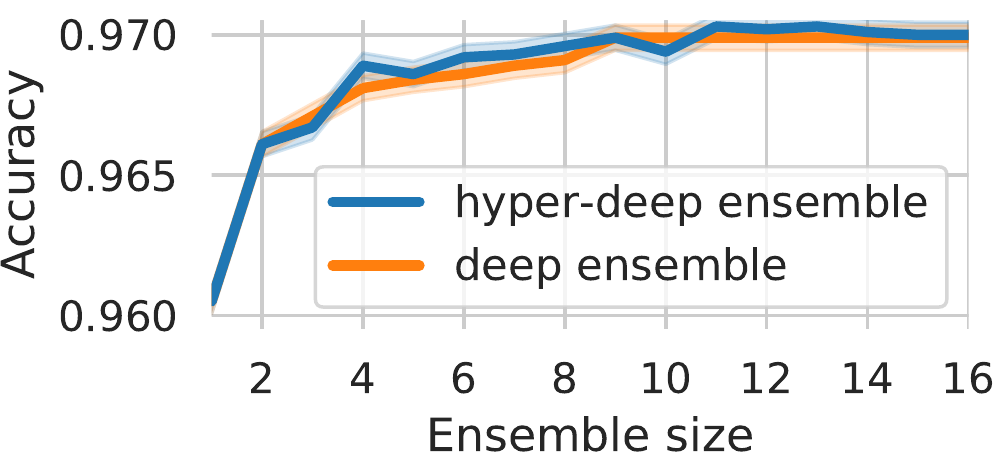}
\includegraphics[scale=0.75]{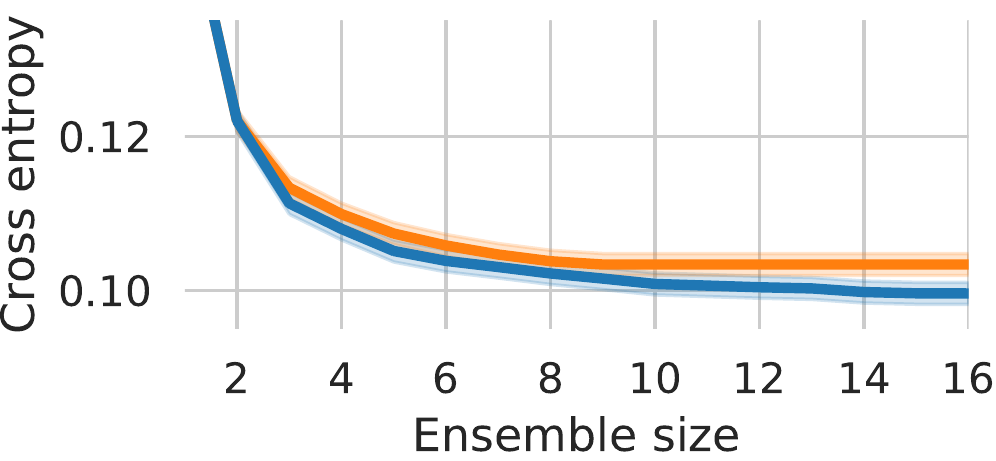}
}
\resizebox{\textwidth}{!}{
\includegraphics[scale=0.75]{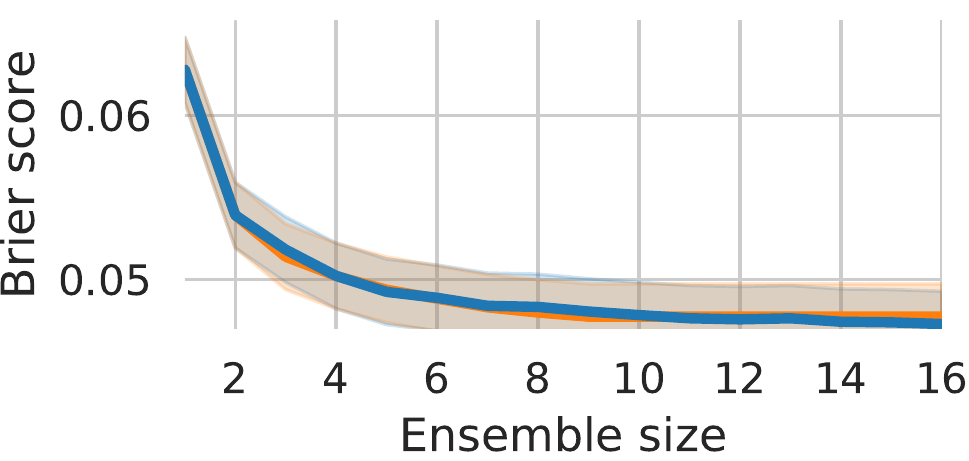}
\includegraphics[scale=0.75]{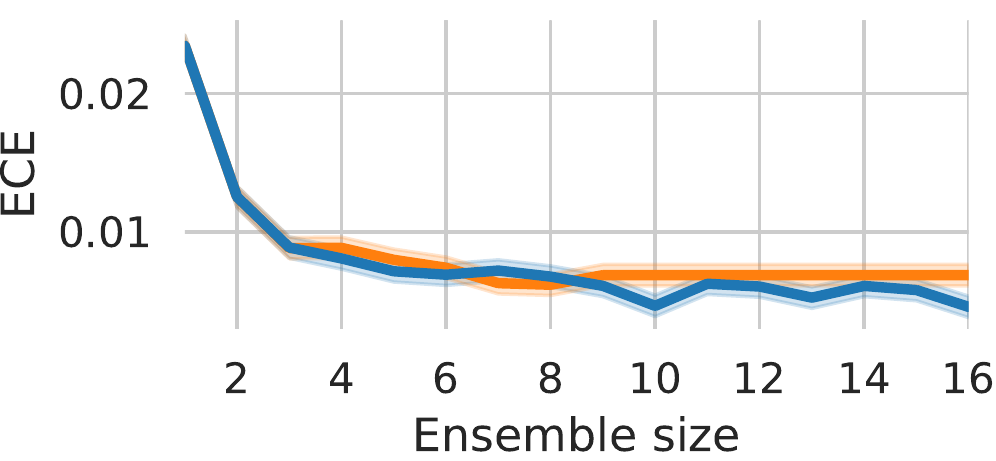}
}
\vspace{-0.5cm}%
\caption{CIFAR-10. Comparison of our hyper-deep ensemble with deep ensemble, for different ensemble sizes in terms of cross entropy (negative log-likelihood), accuracy, Brier score and expected calibration error for a Wide ResNet 28-10 over CIFAR-10.}%
\label{fig:str_hyper_ens_cifar10}%
\vspace{-0.0cm}%
\end{figure}%

\begin{figure}[t]
\vspace{-0.0cm}
\resizebox{\textwidth}{!}{
\includegraphics[scale=0.75]{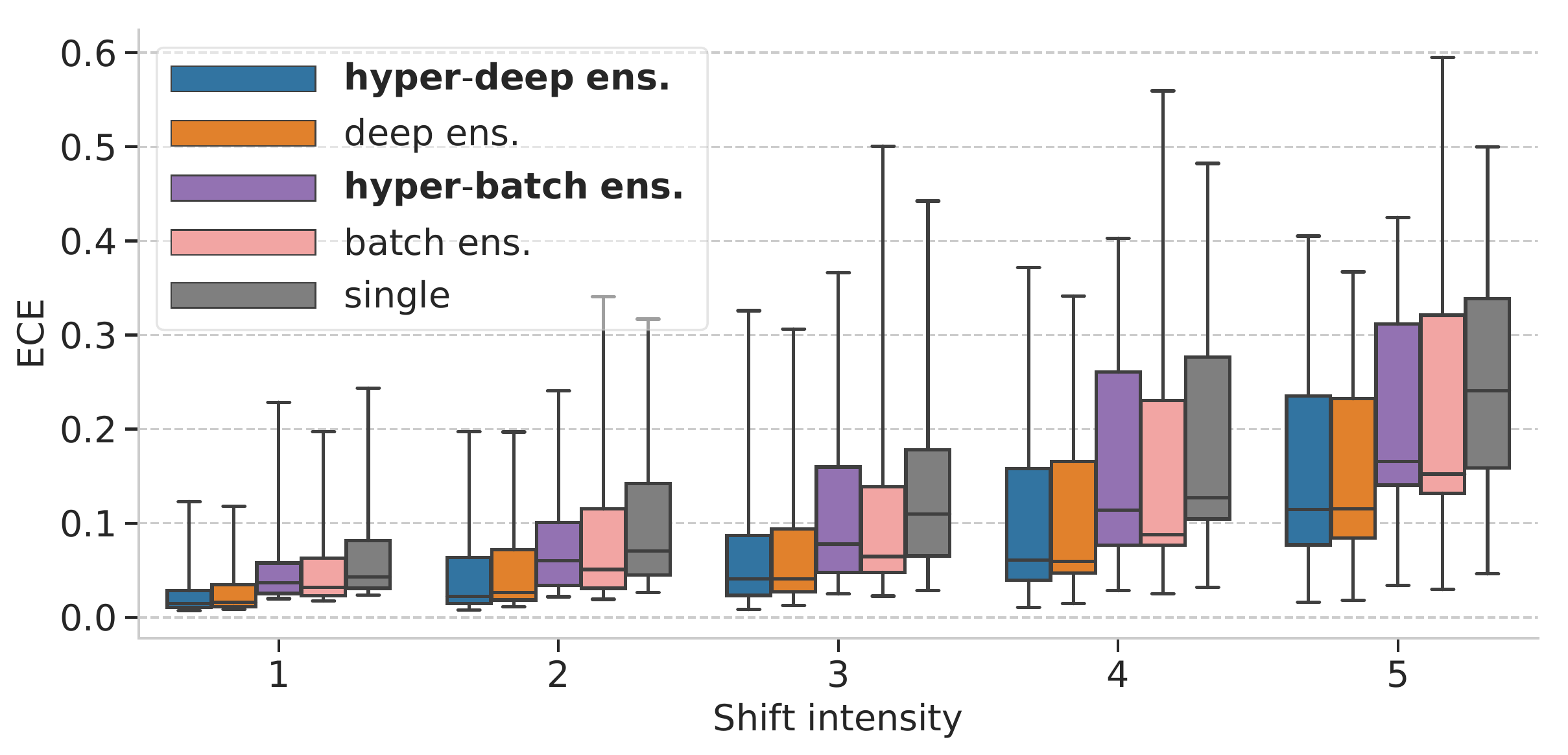}
\includegraphics[scale=0.75]{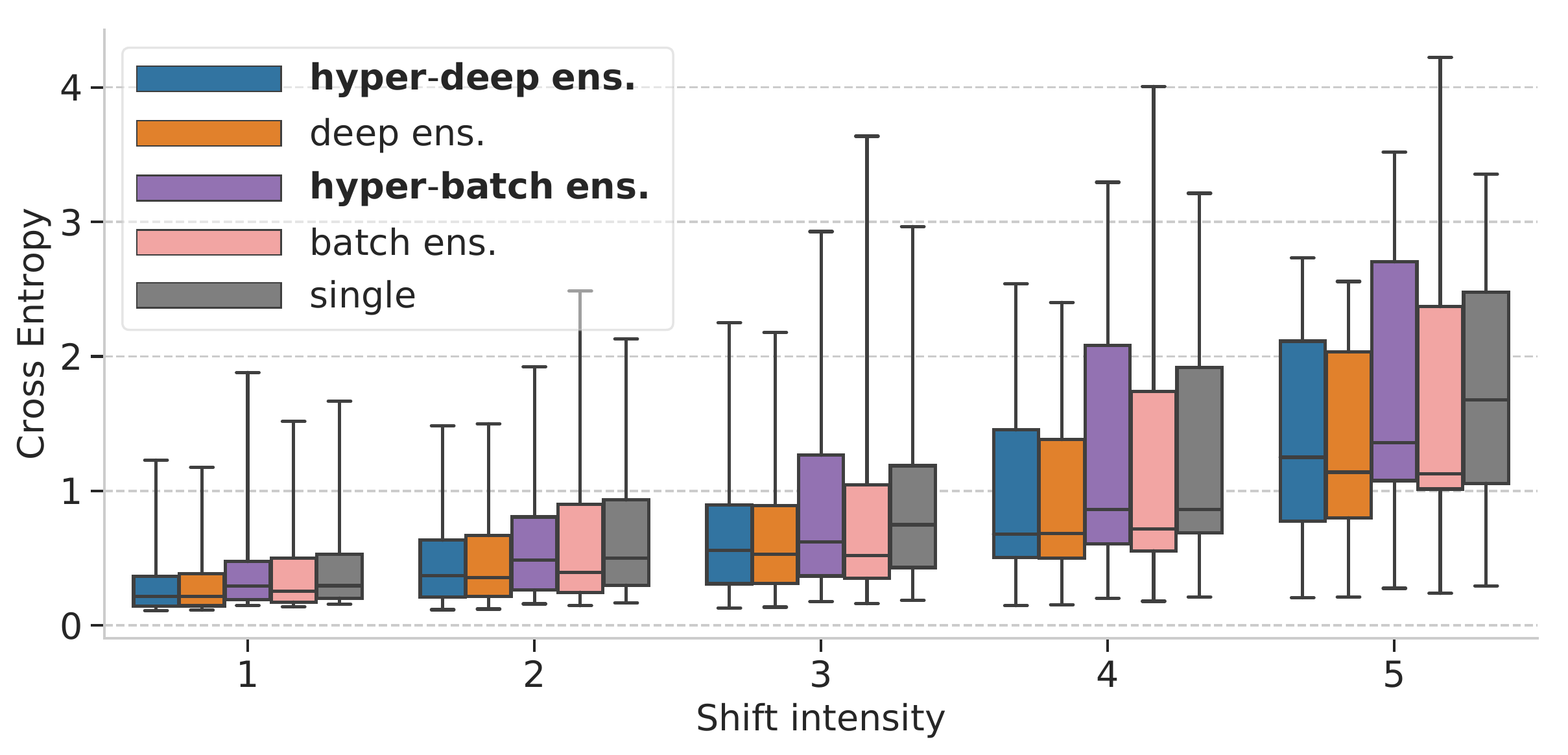}
}
\vspace{-0.5cm}%
\caption{CIFAR-10. Additional plots for calibration on CIFAR-10 corruptions. The boxplots show a comparison of expected calibration error and cross entropy (negative log-likelihood) on different levels of corruption. Each box shows the quartiles summarizing the results across all types of skew while the error bars indicate the min and max across different skew types. The plot for accuracy under corruptions can be found in \Cref{fig:cifar10_ood}.}%
\label{fig:cifar10_ood_appendix}%
\vspace{-0.0cm}%
\end{figure}%

\subsection{Out-of-distribution experiments}
\label{sec:supp_results_full_odd}

In this section, we provide an out-of-distribution evaluation along the line of Table 1 in~\cite{hein2019relu}.
More precisely, for each of the four approaches \deepens{}, \strathyperens{}, \batchens{} and \batchhyperens{}, we compute on out-of-distribution samples from other image datasets the following metrics:
\begin{itemize}
    \item Mean maximum confidence (MMC) on out-distribution samples (lower is better)
    \item The AUC of the ROC curve (AUROC) for the task of discriminating between in- and out-distributions based on the confidence value (higher is better)
    \item The false positive rate at 95\% true positive rate (FPR@95) in the same discriminative task (lower is better).
\end{itemize}
We summarize the results in~\Cref{tab:supp_mat_results_full_odd}, where we consider models both trained on CIFAR-10 (with evaluation on CIFAR-100 and SVHN) and CIFAR-100 (with evaluation on CIFAR-10 and SVHN). In a nutshell, \strathyperens{} (respectively \batchhyperens{}) tends to favourably compare with \deepens{} (respectively \batchens{}) on CIFAR-10 and CIFAR-100, while they appear to perform worse over SVHN.

\begin{table}[t]
\caption{Out-of-distribution evaluation based on other image datasets. The table reports MMC ($\downarrow$)/ AUROC ($\uparrow$) / FPR@95 ($\downarrow$) (see the precise definitions of the metrics in the text).}%
\label{tab:supp_mat_results_full_odd}%
\centering
\vspace*{0.2cm}
\resizebox{\textwidth}{!}{%
\begin{tabular}{@{}lcc|cc@{}}
                & \multicolumn{2}{c|}{\textbf{Trained on CIFAR-100}} & \multicolumn{2}{c}{\textbf{Trained on CIFAR-10}} \\
                & CIFAR-10               & SVHN               & CIFAR-100              & SVHN              \\ \midrule
deep ens (4)      &        0.502 / 0.816 / 0.762               &      0.538 / 0.796 / 0.792             &          0.742 / 0.912 / 0.482             &           0.599 / 0.972 / 0.185        \\
{\color{MidnightBlue}hyper-deep ens (4)}   &           0.524 / 0.822 / 0.741             &           0.580 / 0.787 / 0.787        &          0.730 / 0.915 / 0.469             &       0.608 / 0.967 / 0.237            \\ \midrule
batch ens (4)       &            0.568 / 0.810 / 0.753          &        0.594 / 0.795 / 0.771            &          0.800 / 0.908 / 0.493              &          0.700 / 0.961 / 0.269         \\
{\color{MidnightBlue}hyper-batch ens (4)} &           0.544 / 0.814 / 0.748            &        0.553 / 0.813 / 0.753            &          0.746 / 0.907 / 0.519            &           0.675 / 0.951 / 0.364        \\
\end{tabular}%
}
\vspace*{-0.0cm}
\end{table}

\subsection{Complementary results for CIFAR-10}
\label{sec:supp_results_cifar10}
In this section we show complementary results to those presented in the main text for CIFAR-10. \Cref{fig:str_hyper_ens_cifar10} compares hyper-deep ensembles against deep ensembles for varying ensemble sizes. We find that the performance gain on CIFAR-10 is not as substantial as on CIFAR-100 presented in \Cref{fig:str_hyper_ens_cifar100}. However, \strathyperens{} improves upon \deepens{} for large ensemble sizes in terms of NLL (cross entropy) and expected calibration error (ECE). The accuracy of \strathyperens{} is slightly higher for most ensemble sizes (except for ensemble sizes 3 and 10).

\Cref{fig:cifar10_ood_appendix} shows a comparison of additional metrics on the out of distribution experiment presented in the main text, \Cref{fig:cifar10_ood}. We observe the same trend as in~\Cref{fig:cifar10_ood} that \batchhyperens{} is more robust than \batchens{} as it typically leads to smaller worst values (see top whiskers in the boxplot).

\begin{figure}[t]
\vspace{-0.0cm}
\resizebox{\textwidth}{!}{
\includegraphics[scale=0.75]{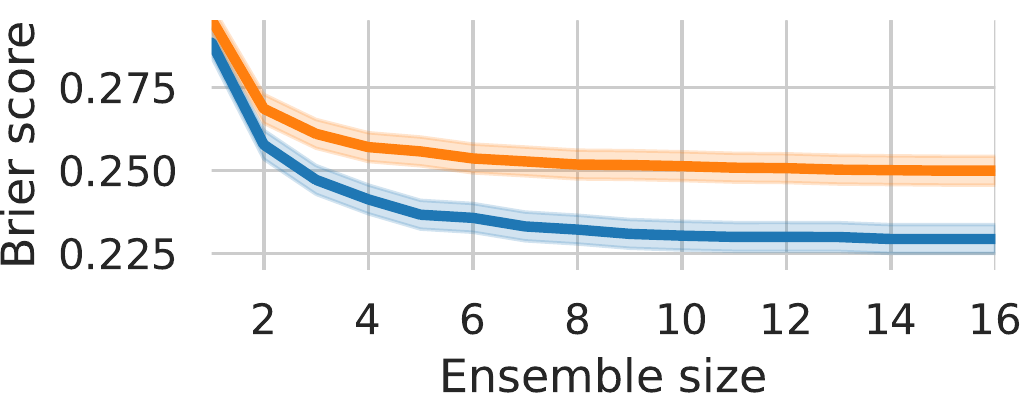}
\includegraphics[scale=0.75]{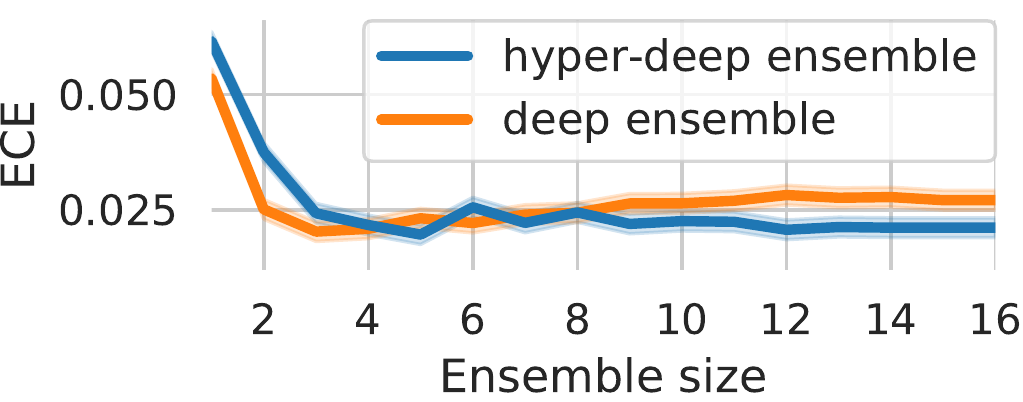}
}
\vspace{-0.5cm}%
\caption{CIFAR-100. Comparison of our hyper-deep ensemble with deep ensemble, for different ensemble sizes, in terms of Brier score and expected calibration error for a Wide ResNet 28-10 over CIFAR-100. Plots for negative log-likelihood and accuracy can be found in the main text, in ~\Cref{fig:str_hyper_ens_cifar100}. }%
\label{fig:str_hyper_ens_cifar100_additional}%
\vspace{-0.0cm}%
\end{figure}%

\subsection{Complementary results for CIFAR-100}
\label{sec:supp_results_cifar100}
In this section we show complementary results to those presented in the main text for CIFAR-100.
\Cref{fig:str_hyper_ens_cifar100_additional} presents additional metrics (Brier score and expected calibration error) for varying ensemble sizes for hyper-deep ensemble and deep ensemble. Additionally to the strong improvements in terms of accuracy and NLL presented in \Cref{fig:str_hyper_ens_cifar100}, we find that \strathyperens{} also improves in terms of Brier score and but is slightly less calibrated than deep ensemble for large ensemble sizes.

\subsection{Memory and training time cost}
\label{sec:supp_time_and_memory}

For hyper-batch ensemble and batch ensemble, \Cref{tab:train_cost} reports the training time and memory cost in terms of number of parameters. Our method is roughly twice as costly as batch ensemble with respect to those two aspects. As demonstrated in the main text, this comes with the advantage of achieving better prediction performance. In \Cref{sec:supp_mat_deep_batch_ens} we show that doubling the number of parameters for batch ensemble still leads to worse performance than our method.

\setlength{\tabcolsep}{4pt}
\begin{table}[t]
\caption{Comparison of the numbers of parameters and training cost for \batchhyperens{} and \batchens{} for Wide ResNet 28-10.}
\label{tab:train_cost}
\begin{center}
{\small
\begin{tabular}{lcccc}
\toprule
 CIFAR-10&  Time/epoch. & total epochs. & total time & \# parameters \\ \midrule \midrule
\batchhyperens{} & 2.07 min.  & 300 & 10.4h  & 73.1M \\
\cmidrule(l){1-5}
\batchens{} & 1.01 min.  & 250 & 4.2h & 36.6M \\
 \toprule
 CIFAR-100 & & & & \\ \midrule \midrule
\batchhyperens{} & 2.16 min. & 300  & 10.8h & 73.2M \\
\cmidrule(l){1-5}
\batchens{} & 1.10 min. & 250 & 4.6h &  36.6M\\
\bottomrule
\end{tabular}%
}
\end{center}
\end{table}

\section{Towards more compact self-tuning layers}\label{sec:supp_mat_more_compact_stn_layers}

The goal of this section is to motivate the introduction of different, more compact parametrizations of the layers in self-tuning networks.

In~\cite{mackay2019self}, the choice of their parametrization (i.e., shifting and rescaling) is motivated by the example of ridge regression whose solution is viewed as a particular 2-layer linear network (see details in Section~B.2 of~\cite{mackay2019self}). The parametrization is however not justified for other losses beyond the square loss. Moreover, by construction, this parametrization leads to at least a 2x memory increase compared to using the corresponding standard layer.

If we take the example of the dense layer with input and output dimensions $r$ and $s$ respectively, recall that we have
\begin{equation*}
    \Wb + \Deltab \circ \eb(\lambdab),
    \ \text{with}\ \Wb, \Deltab \in \Real^{r \times s}.
\end{equation*}
Let us denote by $\zetab_j \in \{0,1\}^{s}$ the one-hot vector with its $j$-th entry equal to 1 and 0 elsewhere, $e_j(\lambdab)$ the $j$-th entry of $\eb(\lambdab)$ and $\deltab_j$ the $j$-th column of $\Deltab$. We can rewrite the above equation as
\begin{equation}\label{eq:supp_mat_rewrite_mackay_layer_as_linear_combination}
    \Wb + \sum_{j=1}^s e_j(\lambdab) \deltab_j \zetab_j^\top
    =\Wb + \sum_{j=1}^s e_j(\lambdab) \Wb_j
    =\sum_{j=0}^s e_j(\lambdab) \Wb_j
\ \text{with}\
e_0(\lambdab) = 1
\ \text{and}\
\Wb_0 = \Wb.
\end{equation}
As a result, we can re-interpret the parametrization of~\cite{mackay2019self} as a very specific linear combination of parameters $\Wb_j$ where the coefficients of the combination, i.e., $\eb(\lambdab)$, depend on $\lambdab$.

Based on this observation and insight, we want to further motivate the use of self-tuned layers with more general linear combinations (dependent on $\lambdab$) of parameters, paving the way for more \textit{compact} parametrizations. For instance, with $\Wb \in \Real^{r \times s}, \Gb \in \Real^{r \times h}$ and $\Hb \in \Real^{s \times h}$ as well as $\eb(\lambdab) \in \Real^h$, we could consider
\begin{equation}\label{eq:supp_mat_example_low_rank_dense}
    \Wb + \sum_{j=1}^h e_j(\lambdab) \gb_j \hb_j^\top = \Wb + (\Gb \circ \eb(\lambdab)) \Hb^\top.
\end{equation}
Formulation~(\ref{eq:supp_mat_example_low_rank_dense}) comes with two benefits.
On the one hand, it reduces the memory footprint, as controlled by the low-rank factor $h$ which impacts the size of both $(\Gb \circ \eb(\lambdab)) \Hb^\top$ and $\eb(\lambdab)$. On the other hand, we can hope to get more expressiveness and flexibility since in (\ref{eq:supp_mat_rewrite_mackay_layer_as_linear_combination}), only the $\deltab_j$'s are learned, while in (\ref{eq:supp_mat_example_low_rank_dense}), both the vector $\gb_j$'s and $\hb_j$'s are learned.     

\subsection{Problem statement}

Along the line of~\cite{mackay2019self}, but with a broader scope, beyond the ridge regression setting, we now provide theoretical arguments to justify the use of such a parametrization. We focus on the linear case with arbitrary convex loss functions. We start by recalling some notation, some of which slightly differ from the rest of the paper.
\paragraph{Notations.}
In the following derivations, we will use
\begin{itemize}
    \item Input point $\xb \in \Real^d$ with target $y$
    \item The distribution over pair $(\xb, y)$ is denoted by $\Pcal$
    \item Domain $\Lambda \subseteq \Real^{m+1}$ of $(m+1)$-dimensional hyperparameter $\lambdab=(\lambda_0, \lambdab_1) \in \Lambda$ (with $\lambdab_1$ of dimension $m$). We split the vector representation to make explicitly appear $\lambda_0$, the regularization parameter, for a reason that will be clear afterwards.
    \item Feature transformation of the input points $\phi: \Real^d \mapsto \Real^k$. When the feature transformation is itself parametrized by some hyperparameters $\lambdab_1$, we write $\phi_{\lambdab_1}(\xb)$.
    \item The distribution over hyperparameters $(\lambda_0, \lambdab_1)$ is denoted by $\Qcal$
    \item Embedding of the hyperparameters $\eb: \Lambda \mapsto \Real^q$
    \item The loss function $\hat{y} \mapsto \ell_{\lambdab_1}(y, \hat{y})$, potentially parameterized by some hyperparameters $\lambdab_1$.
\end{itemize}

We focus on the following formulation
\begin{equation}\label{eq:supp_mat_main_problem}
    \min_{\Ub \in \Real^{k \times q}} \Exp_{(\lambda_0, \lambdab_1) \sim \Qcal} \Big[
    \Exp_{(\xb,y) \sim \Pcal} \Big[ \ell_{\lambdab_1}(y, \phi_{\lambdab_1}(\xb)^\top \Ub \eb(\lambdab)) \Big] + \frac{\lambda_0}{2} \|\Ub \eb(\lambdab)\|^2
    \Big].
\end{equation}
Note the generality of (\ref{eq:supp_mat_main_problem}) where the hyperparameters sampled from $\Qcal$ influence the regularization term (via $\lambda_0$), the loss (via $\ell_{\lambdab_1}$) and the data representation (with $\phi_{\lambdab_1}$).

In a nutshell, we want to show that, for any $\lambdab \in \Lambda$, $\Ub \eb(\lambdab)$---i.e., a linear combination of parameters whose combination depends on $\lambdab$, as in~(\ref{eq:supp_mat_example_low_rank_dense})---can well approximate the solution $\wb(\lambdab)$ of 
\begin{equation*}
    \min_{\wb \in \Real^{k}}
    \Exp_{(\xb,y) \sim \Pcal} \Big[ \ell_{\lambdab_1}(y, \phi_{\lambdab_1}(\xb)^\top \wb) \Big] + \frac{\lambda_0}{2} \|\wb\|^2.
\end{equation*}
In Proposition~\ref{prop:supp_mat_approx}, we show that when we apply a stochastic optimization algorithm to~(\ref{eq:supp_mat_main_problem}), e.g., SGD or variants thereof, with solution $\hat{\Ub}$, it holds in expectation over $\lambdab\sim \Qcal$ that $\wb(\lambdab) \approx \hat{\Ub} \eb(\lambdab)$ under some appropriate assumptions.

Our analysis operates with a fixed feature transformation $\phi_{\lambdab_1}$ (e.g., a pre-trained network) and with a fixed embedding of the hyperparameters $\eb$ (e.g., a polynomial expansion). In practice, those two quantities would however be learnt simultaneously during training. We stress that, despite those two technical limitations, the proposed analysis is more general than that of~\cite{mackay2019self}, in terms of both the loss functions and the hyperparameters covered (in~\cite{mackay2019self}, only the squared loss and $\lambda_0$ are considered). 

We define (remembering the definition $\lambdab=(\lambda_0, \lambdab_1)$)
\begin{eqnarray*}
g_\lambdab(\wb) &=& \Exp_{(\xb,y) \sim \Pcal} \Big[ \ell_{\lambdab_1}(y, \phi_{\lambdab_1}(\xb)^\top \wb) \Big]\\
f_\lambdab(\Ub\eb(\lambdab)) &=& g_\lambdab(\Ub\eb(\lambdab)) + \frac{\lambda_0}{2} \|\Ub \eb(\lambdab)\|^2\\
F(\Ub) &=& \Exp_{\lambdab \sim \Qcal} \big[f_\lambdab(\Ub\eb(\lambdab))\big] = \Exp_{\lambdab \sim \Qcal} \big[g_\lambdab(\Ub\eb(\lambdab))\big] + \frac{1}{2}\trace(\Ub \Cb \Ub^\top)
\end{eqnarray*}

\subsection{Assumptions}

\begin{enumerate}
    \item[\textbf{(A1)}] For all $\lambdab \in \Lambda$, $g_\lambdab(\cdot)$ is convex and has $L_\lambdab$-Lipschitz continuous gradients.
    \item[\textbf{(A2)}] The matrices $\Cb = \Exp_{\lambdab \sim \Qcal}[\lambda_0 \eb(\lambdab) \eb(\lambdab)^\top]$ and $\Sigmab = \Exp_{\lambdab \sim \Qcal}[\eb(\lambdab) \eb(\lambdab)^\top]$ are positive definite.
\end{enumerate}

\subsection{Direct consequences}
Under the assumptions above, we have the following properties:
\begin{itemize}
    \item For all $\lambdab \in \Lambda$, the problem $$\min_{\wb\in \Real^k} \Big\{g_\lambdab(\wb) + \frac{\lambda_0}{2} \|\wb\|^2\Big\}$$ admits a unique solution which we denote by $\wb(\lambdab)$. Moreover, it holds that
\begin{equation}\label{eq:supp_mat_opt_condition_w}
    \nabla g_\lambdab(\wb(\lambdab)) + \lambda_0 \wb(\lambdab) = \zerob
\end{equation}
    \item $F(\cdot)$ is strongly convex ($\Cb \succ \zerob$) and the problem
    $$
    \min_{\Ub \in \Real^{k \times q}} F(\Ub)
    $$
    admits a unique solution which we denote by $\Ub^\star$.
\end{itemize}

\subsection{Preliminary lemmas}
Before listing some lemmas, we define for any $\lambdab \in \Lambda$ and any $\Deltab(\lambdab) \in \Real^k$
\begin{equation*}
R(\Deltab(\lambdab)) = g_\lambdab(\Deltab(\lambdab) + \wb(\lambdab)) - g_\lambdab(\wb(\lambdab)) - 
\Deltab(\lambdab)^\top \nabla g_\lambdab(\wb(\lambdab)) 
\end{equation*}
which is the residual of the first-order Taylor expansion of $g_\lambdab(\cdot)$ at $\wb(\lambdab)$.
Given Assumption \textbf{(A1)}, it notably holds that
\begin{equation}\label{eq:supp_mat_inequality_residual}
0 \leq R(\Delta(\lambdab)) \leq \frac{L_\lambdab}{2} \| \Deltab(\lambdab) \|_2^2.
\end{equation}
\begin{lemma}
We have for any $\Ub \in \Real^{k \times q}$ and any $\lambdab \in \Lambda$, with $\Delta(\lambdab)=\Ub \eb(\lambdab) - \wb(\lambdab)$, 
\begin{eqnarray*}
f_\lambdab(\Ub \eb(\lambdab)) &=& g_\lambdab(\Ub\eb(\lambdab)) + \frac{\lambda_0}{2} \|\Ub \eb(\lambdab)\|^2\\
 &=& g_\lambdab(\Deltab(\lambdab) + \wb(\lambdab)) + \frac{\lambda_0}{2} \|\Deltab(\lambdab) + \wb(\lambdab)\|^2\\
 &=& f_\lambdab(\wb(\lambda)) + R(\Deltab(\lambdab)) + \frac{\lambda_0}{2} \|\Deltab(\lambdab)\|^2 + \Deltab(\lambdab)^\top [\nabla g_\lambdab(\wb(\lambdab)) + \lambda_0\wb(\lambdab)]\\
&=& f_\lambdab(\wb(\lambda)) + R(\Deltab(\lambdab)) + \frac{\lambda_0}{2} \|\Deltab(\lambdab)\|^2 
\end{eqnarray*}
where in the last line we have used the optimality condition (\ref{eq:supp_mat_opt_condition_w}) of $\wb(\lambdab)$.
\end{lemma}
As a direct consequence, we have the following result:
\begin{lemma}\label{lem:FU1_leq_FU2}
For any $\Ub_1, \Ub_2 \in \Real^{k \times q}$ and defining for any $\lambdab \in \Lambda$, with $\Deltab_j(\lambdab)=\Ub_j \eb(\lambdab) - \wb(\lambdab)$, it holds that
$$
F(\Ub_1) \leq F(\Ub_2)
$$
if and only if
$$
\Exp_{\lambdab \sim \Qcal} \big[ R(\Deltab_1(\lambdab)) + \frac{\lambda_0}{2} \|\Deltab_1(\lambdab)\|^2 \big]
\leq
\Exp_{\lambdab \sim \Qcal} \big[ R(\Deltab_2(\lambdab)) + \frac{\lambda_0}{2} \|\Deltab_2(\lambdab)\|^2 \big].
$$
\end{lemma}

\subsection{Main proposition}
Before presenting the main result, we introduce a key quantity that will drive the quality of our guarantee.
To measure how well we can approximate the family of solutions $\{\wb(\lambdab)\}_{\lambdab \in \Lambda}$ via the choice of $\eb$ and $\Qcal$, we define
\begin{equation*}
    \Ub_\text{app} = \argmin_{\Ub \in \Real^{k \times q}} 
    \Exp_{\lambdab \sim \Qcal} \big[ \| \Ub \eb(\lambdab) - \wb(\lambdab)  \|_2\big]
    \quad\text{and}\quad
    {\Deltab}_\text{app}(\lambdab)={\Ub}_\text{app} \eb(\lambdab) - \wb(\lambdab).
\end{equation*}
The definition is unique since according to \textbf{(A2)}, we have $\Sigmab \succ \zerob$.
\begin{proposition}\label{prop:supp_mat_approx}
Let assume we have an, possibly stochastic, algorithm $\Acal$ such that after $t$ steps of $\Acal$ to optimize~(\ref{eq:supp_mat_main_problem}), we obtain $\Ub_t$ satisfying
$$
\Exp_\Acal[F(\Ub_t)] \leq F(\Ub^\star) + \varepsilon_t^\Acal
$$
for some tolerance $\varepsilon_t^\Acal \geq 0$ depending on both $t$ and the algorithm $\Acal$. Denoting by $\Deltab_t(\lambdab)=\Ub_t \eb(\lambdab) - \wb(\lambdab)$ the gap between the estimated and actual solution $\wb(\lambdab)$ for any $\lambdab \in \Lambda$, 
it holds that
$$
\Exp_{\Acal,\ \lambdab\sim\Qcal} \Big[ R(\Deltab_t(\lambdab)) + \frac{\lambda_0}{2} \|\Deltab_t(\lambdab) \|^2 \Big]
\leq
\Exp_{\lambdab \sim \Qcal} \Big[ R(\Deltab_\textup{app}(\lambdab)) + \frac{\lambda_0}{2}  \|\Deltab_\textup{app}(\lambdab)\|^2 \Big] + \varepsilon_t^\Acal.
$$
In particular, we have:
$$
\Exp_{\Acal,\ \lambdab\sim\Qcal} \Big[ \lambda_0 \|\Ub_t \eb(\lambdab) - \wb(\lambdab) \|^2 \Big]
\leq
\Exp_{\lambdab \sim \Qcal} \Big[ (L_\lambdab + \lambda_0) \|\Deltab_\textup{app}(\lambdab)\|^2 \Big] + \varepsilon_t^\Acal.
$$
\end{proposition}
\begin{proof}
Starting from 
$$
\Exp_\Acal[F(\Ub_t)] \leq F(\Ub^\star) + \varepsilon_t^\Acal
$$
and applying Lemma~\ref{lem:FU1_leq_FU2}, we end up with (the expectation $\Exp_{\Acal}$ does not impact the result of Lemma~\ref{lem:FU1_leq_FU2} since the term $f_\lambdab(\wb(\lambdab))$ that cancels out on both sides is not affected by $\Acal$)
$$
\Exp_{\Acal,\ \lambdab\sim\Qcal} \big[ R(\Deltab_t(\lambdab)) + \frac{\lambda_0}{2} \|\Deltab_t(\lambdab)\|^2 \big]
\leq
\Exp_{\lambdab \sim \Qcal} \big[ R(\Deltab^\star(\lambdab)) + \frac{\lambda_0}{2} \|\Deltab^\star(\lambdab)\|^2 \big] + \varepsilon_t^\Acal.
$$
Similarly, by definition of $\Ub^\star$ as the minimum of $F(\cdot)$, we have
$$
F(\Ub^\star) \leq F(\Ub_\textup{app})
$$
which leads to
$$
\Exp_{\lambdab\sim\Qcal} \big[ R(\Deltab^\star(\lambdab)) + \frac{\lambda_0}{2} \|\Deltab^\star(\lambdab)\|^2 \big]
\leq
\Exp_{\lambdab \sim \Qcal} \big[ R(\Deltab_\textup{app}(\lambdab)) + \frac{\lambda_0}{2} \|\Deltab_\textup{app}(\lambdab)\|^2 \big].
$$
Chaining the two inequalities leads to the first conclusion. The second conclusion stems from the application of~(\ref{eq:supp_mat_inequality_residual}).
\end{proof}

\end{document}